\documentclass[11pt, letterpaper, copyright, onecolumn]{preprint}

\usepackage[all]{hypcap}
\usepackage{xspace}
\usepackage[capitalize,noabbrev]{cleveref}
\usepackage{subcaption}
\usepackage{wrapfig}
\usepackage{lipsum}
\usepackage{tabularx}
\usepackage{multirow}
\usepackage{array}
\newcolumntype{Y}{>{\centering\arraybackslash}X}
\usepackage{makecell}
\usepackage[export]{adjustbox}
\makeatletter
\def\mathcolor#1#{\@mathcolor{#1}}
\def\@mathcolor#1#2#3{%
  \protect\leavevmode
  \begingroup
    \color#1{#2}#3%
  \endgroup
}
\makeatother
\usepackage[toc,page]{appendix}

\usepackage{listings}

\usepackage[most]{tcolorbox}
\usepackage{enumitem}

\let\oldtexttt\texttt
\renewcommand{\texttt}[1]{\oldtexttt{\small#1}}

\newtheorem{theorem}{Theorem}[section]
\newtheorem{proposition}[theorem]{Proposition}
\newtheorem{lemma}[theorem]{Lemma}

\newtheorem{definition}[theorem]{Definition}

\title{Deep Reinforcement Learning and\\The Tale of Two Temporal Difference Errors}

\runningtitle{Deep Reinforcement Learning and The Tale of Two Temporal Difference Errors}

\author[1]{Juan Sebastian Rojas}
\author[1]{Chi-Guhn Lee}

\affil[1]{University of Toronto, Canada} 

\correspondingauthor{Juan Sebastian Rojas (\href{mailto:juan.rojas@mail.utoronto.ca}{juan.rojas@mail.utoronto.ca}), Chi-Guhn Lee (\href{mailto:cglee@mie.utoronto.ca}{cglee@mie.utoronto.ca})}

\begin{abstract}
The temporal difference (TD) error was first formalized in \citet{Sutton1988-vs}, where it was first characterized as the difference between temporally successive predictions, and later, in that same work, formulated as the difference between a bootstrapped target and a prediction. Since then, these two interpretations of the TD error have been used interchangeably in the literature, with the latter eventually being adopted as the standard critic loss in deep reinforcement learning (RL) architectures. In this work, we show that these two interpretations of the TD error are not always equivalent. In particular, we show that increasingly nonlinear deep RL architectures can cause these interpretations of the TD error to yield increasingly different numerical values. Then, building on this insight, we show how choosing one interpretation of the TD error over the other can affect the performance of deep RL algorithms that utilize the TD error to compute other quantities, such as with deep differential (i.e., average-reward) RL methods. All in all, our results show that the default interpretation of the TD error as the difference between a bootstrapped target and a prediction does not always hold in deep RL settings.
\end{abstract}

\begin{document}

\maketitle

\section{Introduction}
\label{introduction}
One of the defining aspects of reinforcement learning (RL) is that of the temporal difference (TD) error. The concept of the TD error was first formalized in \citet{Sutton1988-vs}, where it was first characterized as the difference between temporally successive predictions, and later, in that same work, formulated as the difference between a bootstrapped target and a prediction. Since then, these two seemingly-equivalent interpretations of the TD error have been used interchangeably in the literature, where the latter interpretation is what often comes to mind when one thinks of the TD error.

When the first modern deep RL methods were proposed (e.g. \citet{Mnih2015-un}), they followed prior work (e.g. \citet{Tesauro1992-ul}) in adapting the interpretation of the TD error as the difference between a bootstrapped target and a prediction as the critic (i.e., value network) loss. This TD-based loss has become synonymous with the TD error in deep RL settings, and has been adopted in modern value-based (e.g. \citet{Mnih2015-un}) and actor-critic (e.g. \citet{Haarnoja2018-xd}) deep RL architectures. Indeed, if one were asked to compute the TD error at a given time step, \(\delta_t\), for a deep RL method, most would do so as the \(B\)-sized batch average difference between bootstrapped targets and estimates. For example, in a DQN-style algorithm (we will formalize our notation in Section \ref{prelim}):
\begin{equation}
\label{eq_tde_0}
    \delta_t \approx \frac{1}{B}\sum_{j=1}^{B}\left[R_{j,t+1} + \gamma \max_{a_*}Q(S_{j, t+1},a_*) - Q(S_{j,t}, A_{j,t})\right].
\end{equation}

But how reliable is this estimate of the TD error? And how does it compare to the original interpretation of the TD error as the difference between temporally successive predictions? In other words, is the output of Equation \eqref{eq_tde_0} always an accurate and appropriate representation of the TD error?

In this work, we seek to answer these questions. In particular, we show that increasingly nonlinear deep RL architectures can cause the two interpretations of the TD error to yield increasingly different numerical values. While we are hardly the first to notice a discrepancy between the two interpretations of the TD error in function approximation settings (e.g. see Exercise 9.6 in \citet{Sutton2018-eh}), to the best of our knowledge, this work is the first to provide a formal exploration and characterization of the differences between the two interpretations of the TD error.

Importantly, we are able to leverage insights from this exploration to show how choosing one interpretation of the TD error over the other can affect the performance of deep RL algorithms that utilize the TD error to compute other quantities, such as with deep differential (i.e., average-reward) RL methods. All in all, our results show that the default interpretation of the TD error as the difference between a bootstrapped target and a prediction does not always hold in deep RL settings.

\section{Preliminaries}
\label{prelim}
The agent-environment interaction that occurs in reinforcement learning is often modeled as a Markov decision process (MDP) \citep{Puterman1994-dq}. More formally, a finite discounted MDP is the tuple \(\mathcal{M} \doteq \langle \mathcal{S}, \mathcal{A}, \mathcal{R}, p , \gamma\rangle\), where \(\mathcal{S}\) is a finite set of states, \(\mathcal{A}\) is a finite set of actions, \(\mathcal{R} \subset \mathbb{R}\) is a finite set of rewards, \(p: \mathcal{S}\, \times\, \mathcal{A}\, \times\, \mathcal{R}\, \times\,  \mathcal{S} \rightarrow{} [0, 1]\) is a probabilistic transition function that describes the dynamics of the environment, and \(\gamma \in [0,1]\) is the discount factor. At each discrete time step, \(t = 0, 1, 2, \ldots\), an agent chooses an action, \(A_t \in \mathcal{A}\), based on its current state, \(S_t \in \mathcal{S}\), and receives a reward, \(R_{t+1} \in \mathcal{R}\), while transitioning to a (potentially) new state, \(S_{t+1}\), such that \(p(s', r \mid s, a) = \mathbb{P}(S_{t+1} = s', R_{t+1} = r \mid S_t = s, A_t = a)\). In a discounted MDP, an agent aims to find a policy, \(\pi: \mathcal{S} \rightarrow{} \mathcal{A}\), that optimizes the expected discounted sum of rewards. The performance of a policy can be characterized via a \emph{value function}, such as the state-action value function, \( Q_\pi \):
\begin{equation}
\label{eq_disc_value}
Q_{\pi}(s, a) \doteq \mathbb{E}_{\pi}\!\left[ \sum_{t=0}^{T} \gamma^{t} R_{t+1} \mid S_0 = s, A_0 = a, A_{1:T} \sim \pi \right],
\end{equation}
where \(T\) is the length of the time horizon of the MDP. In this work, we limit our discussion to \emph{stationary Markov} policies, which are time-independent policies that satisfy the Markov property.

In the context of RL, the state-action value function can be optimized via Q-learning \citep{Watkins1992-nq}. A tabular version of the algorithm is shown below:
\begin{subequations}
\label{eq_q_learning_tabular}
\begin{align}
\label{eq_q_learning_tabular_1}
& \delta_{t} = R_{t+1} + \gamma \max_{a_* \in \mathcal{A}}Q_{t}(S_{t+1}, a_*) - Q_{t}(S_t, A_t)\\
\label{eq_q_learning_tabular_2}
& Q_{t+1}(S_t, A_t) = Q_{t}(S_t, A_t) + \alpha_{t}\delta_{t}\\
\label{eq_q_learning_tabular_3}
& Q_{t+1}(s, a) = Q_{t}(s, a), \quad \forall (s,a) \neq (S_t, A_t) \, ,
\end{align}
\end{subequations}
where \(Q_t: \mathcal{S} \times \mathcal{A} \rightarrow \mathbb{R}\) is a table of state-action value function estimates, \(\delta_t\) is the TD error, and \(\alpha_t\) is the value function step size at time \(t\). Given some assumptions, Q-learning has been shown to converge to the optimal state-action value function in tabular settings \citep{Watkins1992-nq}.

In deep RL settings, the value function is approximated with a neural network, \( Q_{\boldsymbol{w}}(s,a) \), where \( \boldsymbol{w} \) denotes the parameters of the network. The agent learns the parameters of this \emph{critic} network by minimizing the \emph{mean square value loss} via gradient updates as follows \citep{Mnih2015-un}:
\begin{subequations}
\begin{align}
\label{eq_dqn_1}
\mathcal{L}_{\text{critic}}(\boldsymbol{w}_t)
&\doteq \frac{1}{B} \sum_{j=1}^{B}
\frac{1}{2}
\left(
R_{j, t+1}
+ \gamma \max_{a_* \in \mathcal{A}} Q_{\boldsymbol{w}_t}(S_{j, t+1},a_*)
- Q_{\boldsymbol{w}_t}(S_{j,t},A_{j,t})
\right)^2
\\
\label{eq_dqn_2}
\boldsymbol{w}_{t+1}
&= \boldsymbol{w}_{t}
- \alpha_t \nabla_{\boldsymbol{w}_t} \mathcal{L}_{\text{critic}} \, ,
\end{align}
\end{subequations}
where \(\{(S_{j,t}, A_{j,t}, R_{j,t+1}, S_{j,t+1})\}_{j=1}^{B}\) denotes a \(B\)-sized batch of agent-environment interactions sampled from a replay buffer, and \(\nabla_{\boldsymbol{w}_t} \mathcal{L}_{\text{critic}}\) denotes the gradient of the loss \eqref{eq_dqn_1} with respect to the network parameters. This kind of value function parameterization has been utilized in many value-based (e.g. \citet{Mnih2015-un}) and actor-critic (e.g. \citet{Haarnoja2018-xd}) deep RL algorithms, and has been shown to yield strong empirical performance in challenging learning environments.

\section{The Tale of Two Temporal Difference Errors}
\label{two_td_errors}
In this section, we present our primary contribution: a formal exploration and characterization of the differences between the two interpretations of the TD error. We first perform this characterization from an analytical perspective, followed by a characterization from an empirical perspective. Importantly, we note that we are hardly the first to notice a discrepancy between the two interpretations of the TD error in function approximation settings. For example, such a discrepancy is hinted at in Exercise 9.6 of \citet{Sutton2018-eh}. However, to the best of our knowledge, this work represents the first formal characterization of the differences between the two interpretations of the TD error. In particular, we will show that increasingly nonlinear deep RL architectures can cause the two interpretations of the TD error to yield increasingly different numerical values.

We begin our characterization by formalizing the two interpretations of the TD error. We begin with the common interpretation of the TD error as the difference between a bootstrapped target and a prediction. We will refer to this interpretation as the \emph{explicit} TD error, \(\delta^e\). For example, in the tabular Q-learning algorithm \eqref{eq_q_learning_tabular}, the explicit TD error at time \(t\) is defined as follows:
\begin{equation}
\label{eq_tde_1}
\delta^{e}_{t} = R_{t+1} + \gamma \max_{a_*}Q_{t}(S_{t+1}, a_*) - Q_{t}(S_t, A_t),
\end{equation}
where, in this case (i.e., for Q-learning), \(R_{t+1} + \gamma \max_{a_*}Q_{t}(S_{t+1}, a_*)\) denotes the bootstrapped target, and \(Q_{t}(S_t, A_t)\) denotes the prediction at time \(t\). Analogous explicit TD errors can be derived for other tabular RL algorithms using different targets and estimates. For simplicity, we will use Q-learning as a concrete example throughout this work, however, we emphasize that our definition of the explicit TD error encompasses all TD errors defined based on the difference between a bootstrapped target and a prediction. This is formalized as Definition \ref{defn_tde} below:
\vspace{4pt}
\begin{definition}[Explicit TD Error]
\label{defn_tde}
An explicit TD error, \(\delta^e\), is any TD error defined based on the difference between a bootstrapped target and a prediction.
\end{definition}

We now proceed with the interpretation of the TD error as the difference between temporally successive predictions. We will refer to this interpretation as the \emph{implicit} TD error, \(\delta^i\). For example, in the tabular Q-learning algorithm \eqref{eq_q_learning_tabular}, the implicit TD error at time \(t\) is defined as follows:
\begin{equation}
\label{eq_tdi_1}
\delta^{i}_{t} = \frac{1}{\alpha_t}\left(Q_{t+1}(S_t, A_t) - Q_{t}(S_t, A_t)\right).
\end{equation}
Another way to think about the implicit TD error \eqref{eq_tdi_1} is by taking Equation \eqref{eq_q_learning_tabular_2} and ``solving'' for the TD error. Again, we will use Q-learning as a concrete example throughout this work, however, our definition of the implicit TD error encompasses all TD errors defined based on the difference between temporally successive predictions. This is formalized as Definition \ref{defn_tdi} below:
\vspace{4pt}
\begin{definition}[Implicit TD Error]
\label{defn_tdi}
An implicit TD error, \(\delta^i\), is any TD error defined based on the difference between temporally successive predictions.
\end{definition}

Having formally established the notions of the explicit and implicit TD errors, we can now start our characterization of the differences between the two. In particular, when we look at the tabular Q-learning algorithm \eqref{eq_q_learning_tabular}, it is easily seen that \(\delta^{e}_{t} = \delta^{i}_{t}\). This is formalized as Lemma \ref{lemma_1} below:
\vspace{4pt}
\begin{lemma}
\label{lemma_1}
In tabular settings, both interpretations of the TD error are always equivalent. That is, \(\delta^{e}_{t} = \delta^{i}_{t} \; \forall t \in \mathbb{N}\).
\end{lemma}
\begin{proof}
This follows directly from the construction of tabular RL algorithms.
\end{proof}

Indeed, in tabular settings, the two interpretations of the TD error are always equivalent. However, we will now show that this is not always the case in function approximation settings, particularly with increasingly nonlinear deep RL architectures.

To this end, let us begin by extending our definitions of the tabular Q-learning TD errors \eqref{eq_tde_1} and \eqref{eq_tdi_1} into function approximation settings. In particular, in such RL settings, the explicit TD error at time \(t\) can be defined as follows:
\begin{equation}
\label{eq_tde_3}
\delta^{e}_t = \frac{1}{B}\sum_{j=1}^{B}\left[R_{j,t+1} + \gamma \max_{a_*}Q_{\boldsymbol{w}_t}(S_{j, t+1},a_*) - Q_{\boldsymbol{w}_t}(S_{j,t}, A_{j,t})\right],
\end{equation}
and the implicit TD error at time \(t\) can be defined as follows:
\begin{equation}
\label{eq_tdi_3}
\delta^{i}_t = \frac{1}{B}\sum_{j=1}^{B}\left[\frac{1}{\alpha_t}\left(Q_{\boldsymbol{w}_{t+1}}(S_{j,t}, A_{j,t}) - Q_{\boldsymbol{w}_t}(S_{j,t}, A_{j,t})\right)\right],
\end{equation}
where we interpret these TD errors as the \(B\)-sized batch average difference between relevant terms.

We are now ready to pose the central question that we seek to answer in this work: 
\begin{quote}
\centering
\textit{When are \(\delta^{e}_t\) \eqref{eq_tde_3} and \(\delta^{i}_t\) \eqref{eq_tdi_3} equivalent?}
\end{quote}

\subsection{The Two TD Errors Are Not Always Equivalent Even in Linear Function Approximation Settings}
\label{linear_fa}
In Lemma \ref{lemma_1}, we showed that the two interpretations of the TD error are always equivalent in tabular settings. In this section, we will show that even in the simplest kind of function approximation settings, this equivalence begins to break down.

To this end, let us begin by considering the explicit and implicit TD errors in the linear function approximation setting (with a batch size of 1):
\begin{subequations}
\label{eq_tde_4}
\begin{align}
\delta^{e}_t &= R_{t+1} + \gamma \max_{a_*}Q_{\boldsymbol{w}_t}(S_{t+1},a_*) - Q_{\boldsymbol{w}_t}(S_{t}, A_{t})\\
& = R_{t+1} + \gamma \max_{a_*}\boldsymbol{w}_{t}^{T}\boldsymbol{x}(S_{t+1},a_*) - \boldsymbol{w}_{t}^{T}\boldsymbol{x}(S_{t}, A_{t}),
\end{align}
\end{subequations}
and
\begin{subequations}
\label{eq_tdi_4}
\begin{align}
\delta^{i}_t &= \frac{1}{\alpha_t}\left(Q_{\boldsymbol{w}_{t+1}}(S_{t}, A_{t}) - Q_{\boldsymbol{w}_t}(S_{t}, A_{t})\right)\\
& = \frac{1}{\alpha_t}\left(\boldsymbol{w}_{t+1}^{T}\boldsymbol{x}(S_{t}, A_{t}) - \boldsymbol{w}_{t}^{T}\boldsymbol{x}(S_{t}, A_{t})\right),
\end{align}
\end{subequations}
where \(\boldsymbol{x}(S_t,A_t)\) denotes a feature vector with the same dimensions as the network parameter vector, \(\boldsymbol{w}\), for the state-action pair \((S_t, A_t)\). Next, as per Equation \eqref{eq_dqn_2}, we can write the standard \emph{semi-gradient} update for \(\boldsymbol{w}_{t+1}\) as follows \citep{Sutton1988-vs, Sutton2018-eh}:
\begin{subequations}
\label{eq_w_update}
\begin{align}
\label{eq_w_update_1}
\boldsymbol{w}_{t+1} &= \boldsymbol{w}_{t}
- \alpha_t \nabla_{\boldsymbol{w}_t}
\frac{1}{2}
\left(
R_{t+1}
+ \gamma \max_{a_* \in \mathcal{A}} Q_{\boldsymbol{w}_t}(S_{t+1},a_*)
- Q_{\boldsymbol{w}_t}(S_{t},A_{t})
\right)^2\\
\label{eq_w_update_2}
& \approx \boldsymbol{w}_{t}
+ \alpha_t \delta^{e}_t \boldsymbol{x}(S_{t}, A_{t}).
\end{align}
\end{subequations}

Finally, we can combine Equations \eqref{eq_tde_4}, \eqref{eq_tdi_4}, and \eqref{eq_w_update_2} as follows:
\begin{subequations}
\label{eq_tdi_vs_tde_linear}
\begin{align}
\label{eq_tdi_vs_tde_linear_1}
\delta^{i}_t &= \frac{1}{\alpha_t}\left(\boldsymbol{w}_{t+1}^{T}\boldsymbol{x}(S_{t}, A_{t}) - \boldsymbol{w}_{t}^{T}\boldsymbol{x}(S_{t}, A_{t})\right)\\
\label{eq_tdi_vs_tde_linear_2}
& = \frac{1}{\alpha_t}\left((\boldsymbol{w}_{t}
+ \alpha_t \delta^{e}_t \boldsymbol{x}(S_{t}, A_{t}))^{T}\boldsymbol{x}(S_{t}, A_{t}) - \boldsymbol{w}_{t}^{T}\boldsymbol{x}(S_{t}, A_{t})\right)\\
\label{eq_tdi_vs_tde_linear_3}
& = \delta^{e}_t \boldsymbol{x}(S_{t}, A_{t})^{T}\boldsymbol{x}(S_{t}, A_{t}).
\end{align}
\end{subequations}

Let us examine the implications of Equation \eqref{eq_tdi_vs_tde_linear_3}. We can see that the two TD errors are only equal to each other in the linear function approximation case if \(\boldsymbol{x}(S_{t}, A_{t})^{T}\boldsymbol{x}(S_{t}, A_{t}) \doteq ||\boldsymbol{x}(S_{t}, A_{t})||^2 = 1\), or if  \(\delta^e_t = \delta^i_t = 0\). We note that this result is consistent with Lemma \ref{lemma_1}, given that it is well-known that tabular RL methods can be viewed as a special case of linear function approximation with a feature vector that satisfies \(||\boldsymbol{x}(S_{t}, A_{t})||^2 = 1\).

As such, we have shown that even with linear function approximation, the implicit and explicit TD errors are not always equal to each other. We note that this result was hinted at in Exercise 9.6 of \citet{Sutton2018-eh}, which is formalized as Lemma \ref{lemma_2} below:
\vspace{4pt}
\begin{lemma}[Similar to Exercise 9.6 of \citet{Sutton2018-eh}]
\label{lemma_2}
In the linear function approximation setting, given an update batch size of 1, the explicit and implicit TD errors are only equal to each other at time \(t\) if and only if the feature vector satisfies \(||\boldsymbol{x}(S_{t}, A_{t})||^2 = 1\), or if at time \(t\) the RL algorithm has converged to a (potentially suboptimal) solution, such that \(\delta^e_t = \delta^i_t = 0\).
\end{lemma}

\subsection{Nonlinearity Can Cause Further Divergence Between the Two TD Errors}
\label{nonlinear_fa}
In the previous section, we saw that even in the linear function approximation case the two TD errors are not always equal to each other. In this section, we will show that adding even just a single nonlinearity into the function approximator complicates the picture even further. In particular, let us consider the effect of adding a single nonlinear activation function to the linear function approximator described in Section \ref{linear_fa} (while still maintaining a batch size of 1), such that \(Q_{\boldsymbol{w}_{t}}(S_{t}, A_{t}) = \sigma(\boldsymbol{w}_{t}^{T}\boldsymbol{x}(S_{t}, A_{t}))\), where \(\sigma\) denotes the nonlinear activation function. Applying a similar process to the one shown in Equation \eqref{eq_tdi_vs_tde_linear} yields:
\begin{subequations}
\label{eq_tdi_vs_tde_nonlinear}
\begin{align}
\label{eq_tdi_vs_tde_nonlinear_1}
\delta^{i}_t &= \frac{1}{\alpha_t}\left[\sigma(\boldsymbol{w}_{t+1}^{T}\boldsymbol{x}(S_{t}, A_{t})) - \sigma(\boldsymbol{w}_{t}^{T}\boldsymbol{x}(S_{t}, A_{t}))\right]\\
\label{eq_tdi_vs_tde_nonlinear_2}
& = \frac{1}{\alpha_t}\left[\sigma \left((\boldsymbol{w}_{t}
+ \alpha_t \delta^{e}_t \nabla_{\boldsymbol{w}_t}\sigma(\boldsymbol{w}_{t}^T\boldsymbol{x}(S_{t}, A_{t})))^{T}\boldsymbol{x}(S_{t}, A_{t})\right) - \sigma(\boldsymbol{w}_{t}^{T}\boldsymbol{x}(S_{t}, A_{t}))\right].
\end{align}
\end{subequations}

As per Equation \eqref{eq_tdi_vs_tde_nonlinear_2}, we can see that even adding a single nonlinear activation function greatly limits our ability to compare the explicit and implicit TD errors from an analytical perspective. That is, while it may be \emph{possible} that for some choice of feature vector and (nonlinear) activation function that the expression on the right-hand side of Equation \eqref{eq_tdi_vs_tde_nonlinear_2} evaluates to \(\delta^{e}_t\), we can no longer provide such guarantees in the general case from an analytical perspective. Moreover, it is not hard to imagine how adding more layers and nonlinearities compounds this analytical divergence, ultimately making an exact (analytical) equivalence between the two TD errors increasingly unlikely.

\subsection{Batch Updates Can Also Induce Divergence Between the Two TD Errors}
\label{batch_updates}
In this section, we explore, from an analytical perspective, the effect of batch updates as it pertains to the equivalence of the two TD errors. Given that we showed in Section \ref{nonlinear_fa} how nonlinearity limits our ability to make an analytical comparison between the two TD errors, we will examine the effect of batch updates in the linear function approximation case. In particular, for a batch update of size \(B\), we can apply a similar process to the one described in Equation \eqref{eq_tdi_vs_tde_linear}, which yields:
\begin{subequations}
\label{eq_tdi_vs_tde_batch} 
\begin{align} 
\label{eq_tdi_vs_tde_batch_1}
\delta^{i}_{t} \doteq \frac{1}{B}\sum_{j=1}^{B}\delta^{i}_{j,t} &= \frac{1}{B}\sum_{j=1}^{B}\frac{1}{\alpha_t}\left(\boldsymbol{w}_{t+1}^{T}\boldsymbol{x}(S_{j,t}, A_{j,t}) - \boldsymbol{w}_{t}^{T}\boldsymbol{x}(S_{j,t}, A_{j,t})\right)\\
\label{eq_tdi_vs_tde_batch_2}
& = \frac{1}{B}\sum_{j=1}^{B}\frac{1}{\alpha_t}\left[\left(\boldsymbol{w}_{t} + \alpha_t \frac{1}{B}\sum_{k=1}^{B} \delta^{e}_{k,t} \boldsymbol{x}(S_{k,t}, A_{k,t})\right)^{T}\boldsymbol{x}(S_{j,t}, A_{j,t}) - \boldsymbol{w}_{t}^{T}\boldsymbol{x}(S_{j,t}, A_{j,t})\right]\\
\label{eq_tdi_vs_tde_batch_3}
& = \frac{1}{B}\sum_{k=1}^{B} \delta^{e}_{k,t} \frac{1}{B} \sum_{j=1}^{B} \boldsymbol{x}(S_{k,t}, A_{k,t})^{T}\boldsymbol{x}(S_{j,t}, A_{j,t}).
\end{align} 
\end{subequations}

In this case, we would consider the explicit and implicit TD errors equivalent if their batch-averaged values are the same. From Equation \eqref{eq_tdi_vs_tde_batch_3}, we can see that this occurs if: 
\begin{equation}
\label{eq_batch_condition}
\frac{1}{B} \sum_{j=1}^{B} \boldsymbol{x}(S_{k,t}, A_{k,t})^{T}\boldsymbol{x}(S_{j,t}, A_{j,t}) = 1, \, \forall k=1, 2, \ldots, B.
\end{equation}
In words, this condition requires that for every sample \(k\) in the batch, the average inner product between its feature vector and all feature vectors in the batch (including itself) equals one. This condition is considerably more restrictive than the single-sample case, where we only required \(||\boldsymbol{x}(S_{t}, A_{t})||^2 = 1\). Importantly, this means that if the samples in a given batch are sufficiently diverse, such that their corresponding feature vectors point in different directions, then these average inner products will deviate from one. Consequently, this shows that even in the linear function approximation case, batch updates introduce an additional source of divergence between the explicit and implicit TD errors. This argument is formalized as Lemma \ref{lemma_3} below:
\vspace{4pt}
\begin{lemma}
\label{lemma_3}
In the linear function approximation setting with batch updates, the explicit and implicit TD errors are only equal to each other at time \(t\) if and only if Equation \eqref{eq_batch_condition} is satisfied, or if at time \(t\) the RL algorithm has converged to a (potentially suboptimal) solution, such that \(\delta^e_t = \delta^i_t = 0\). 
\end{lemma}

\subsection{An Empirical Comparison Between the Two TD Errors}
\label{empirical_compare}
In the previous sections, we showed, from an analytical perspective, how the equivalence between the explicit and implicit TD errors can break down in deep RL settings. In this section, we complement these analytical results by performing an empirical characterization of this divergence. 

To this end, we compared the explicit and implicit TD errors when applying Q-learning algorithms on two RL tasks. In the first task, we used a Q-learning algorithm on the well-known \emph{inverted pendulum} environment. Then, in the second task, we used a DQN algorithm \citep{Mnih2015-un} on the Atari \emph{Breakout} environment \citep{Bellemare2013-ef}. In both experiments, we compared the two TD errors when introducing increased nonlinearity into the function approximator. In particular, in the inverted pendulum experiment, we compared the two TD errors when using a linear function approximator (with no batch updates) vs. a function approximator with two hidden layers and batch updates. Similarly, in the Breakout experiment, we compared the two TD errors when using a value network with 0.4M parameters vs. a value network with 6.7M parameters (in both cases, we used batch updates, and the value network had three convolutional layers and two fully connected layers).

\begin{figure}[htbp]
\centerline{\includegraphics[scale=0.53]{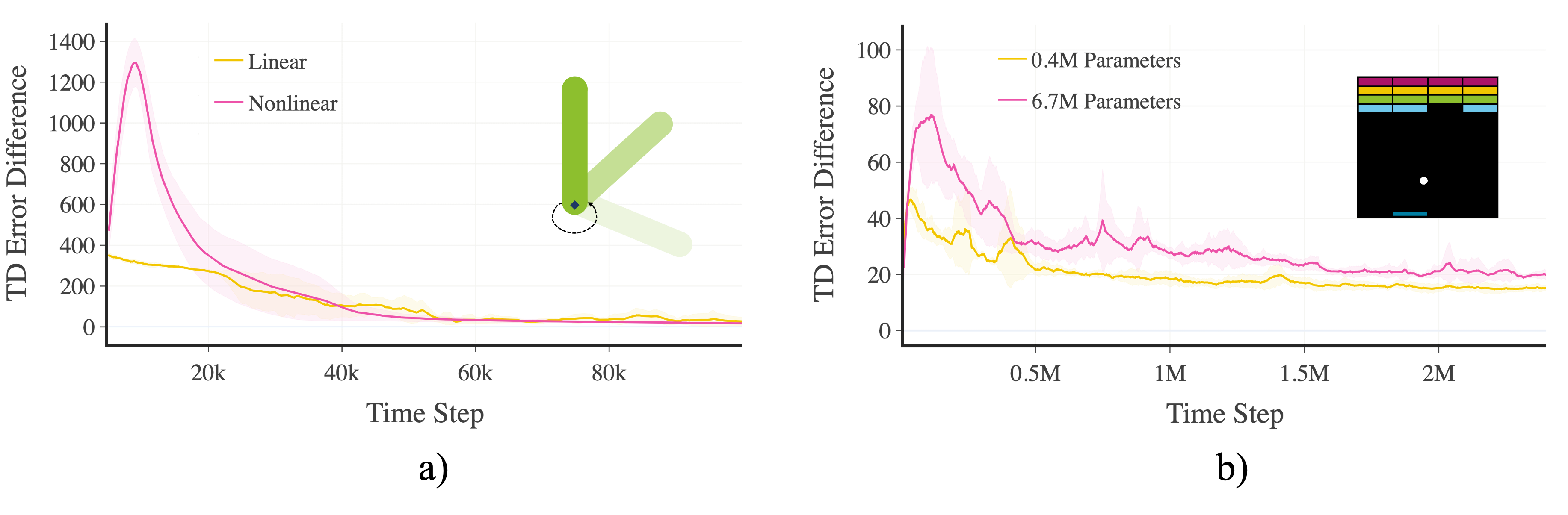}}
\caption{Rolling average of the TD error absolute difference (i.e., \(|\delta^e_t - \delta^i_t|\)) as learning progresses when using \textbf{a)} a Q-learning algorithm in the inverted pendulum environment, and \textbf{b)} a DQN algorithm in the Atari Breakout environment. A solid line denotes the mean TD error absolute difference, and a shaded region denotes a 95\% confidence interval over 4 runs.}
\label{fig_td_diff}
\end{figure}

Figure \ref{fig_td_diff} shows the results of this comparison. In particular, the figure shows the \emph{TD error absolute difference} (i.e., \(|\delta^e_t - \delta^i_t|\)) as learning progresses in both tasks. As shown in Figure \ref{fig_td_diff}\textcolor{mylightblue}{a)}, even in linear function approximation settings, there is still divergence between the explicit and implicit TD errors. We can also see that, in both experiments, this divergence is only intensified by the inclusion of more nonlinearity into the function approximator. As expected, the figure shows that, in both cases, the divergence between the two TD errors becomes smaller as the algorithm converges to a (potentially suboptimal) solution. The full set of experimental details can be found in Appendix \ref{appendix_experiments}.

\section{The Choice of TD Error Matters}
\label{choice_matters}
In the previous section, we showed that increasingly nonlinear deep RL architectures can cause the two interpretations of the TD error to yield increasingly different numerical values. In this section, we leverage this insight to show how choosing one interpretation of the TD error over the other can affect the performance of deep RL algorithms that utilize the TD error to compute other quantities.

In particular, we will focus on evaluating the effect of using the explicit vs. implicit TD error in deep differential Q-learning algorithms \citep{Wan2021-re, Rojas2026-dv}, Q-learning algorithms with value-based reward centering \citep{Naik2024-km}, and A2C algorithms \citep{Mnih2016-tt}. Each of these algorithms utilize the TD error to compute other quantities. In differential and reward centering algorithms, the TD error is used to compute a running estimate of the average of the one-step rewards. In A2C, the TD error is used as a proxy for the advantage function in the actor loss. We discuss the role of the TD error in each class of algorithm in more detail below:

\begin{figure}[htbp]
\centerline{\includegraphics[scale=0.53]{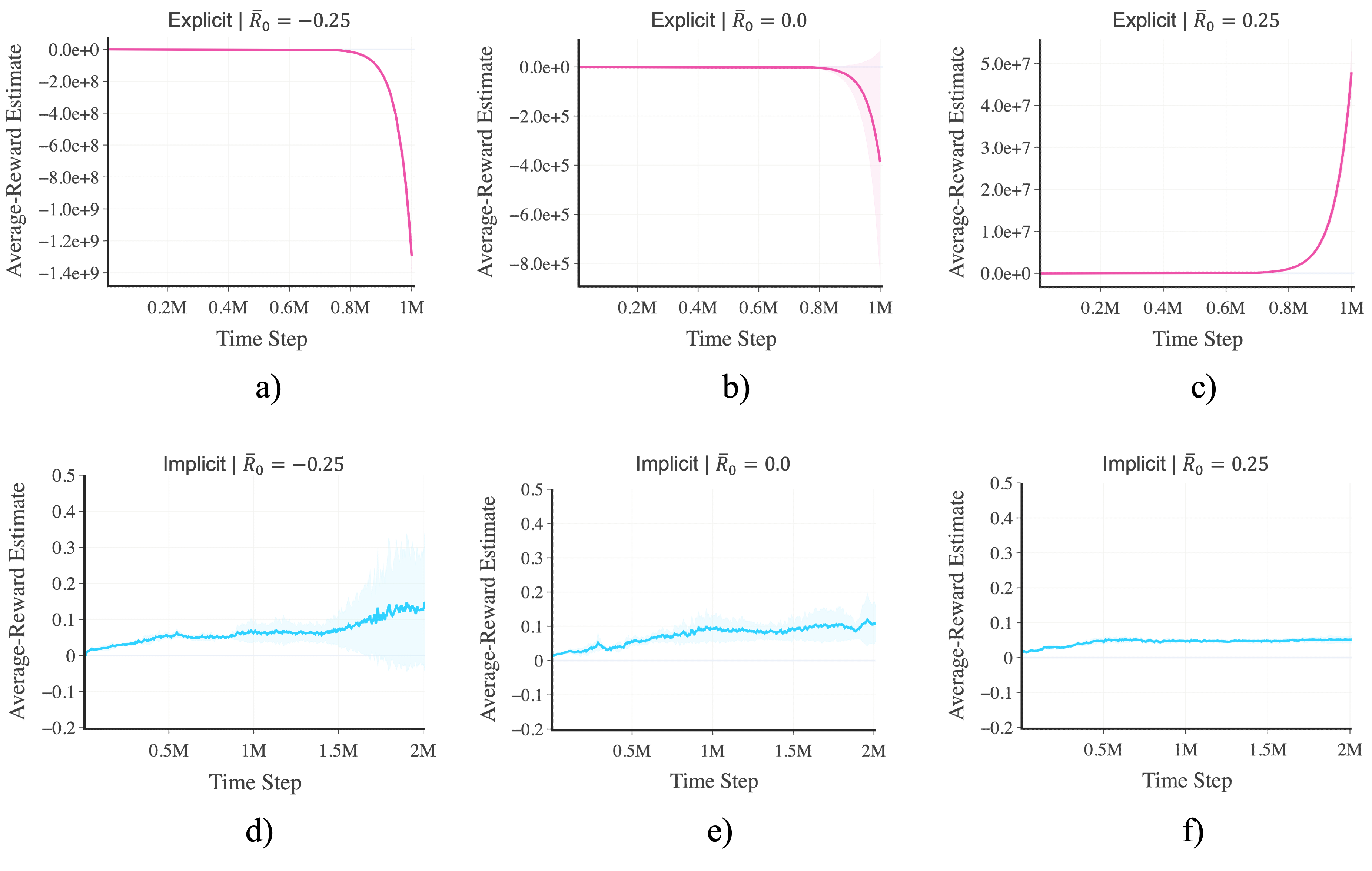}}
\caption{Rolling average-reward estimates when using a deep differential Q-learning algorithm in the Breakout environment. Figures \textbf{a)}, \textbf{b)}, and \textbf{c)} utilize an \emph{explicit} TD error average-reward update with an initial guess of -0.25, 0.0, and 0.25, respectively. Figures \textbf{d)}, \textbf{e)}, and \textbf{f)} utilize an \emph{implicit} TD error update with an initial guess of -0.25, 0.0, and 0.25, respectively. A solid line denotes the mean average-reward estimate, and a shaded region denotes a 95\% confidence interval over 4 runs.}
\label{fig_estimates}
\end{figure}

\subsection{Differential Algorithms and the TD Error}
\label{differential}
In differential RL algorithms \citep{Wan2021-re}, the agent aims to find a policy that optimizes the long-run (or limiting) average-reward, \(\bar{r}\), which is defined as follows for a given policy, \(\pi\):
\begin{equation}
\label{eq_avg_reward}
\bar{r}_{\pi}(s) \doteq  \lim_{n \rightarrow{} \infty} \frac{1}{n} \sum_{t=1}^{n} \mathbb{E}[R_t \mid S_0=s, A_{0:t-1} \sim \pi].
\end{equation}
Typically, differential algorithms invoke \emph{ergodicity-like} assumptions (e.g. see \citet{Wan2021-re}) so that the average-reward objective \eqref{eq_avg_reward} remains well-defined and becomes independent of initial conditions (i.e., \(\bar{r}_{\pi}(s) = \bar{r}_{\pi}\)). To optimize the average-reward \eqref{eq_avg_reward}, differential methods optimize \emph{differential} value functions, such as the state-action differential value function:
\begin{equation}
\label{eq_diff_value_function}
Q_{\pi}(s, a) \doteq \mathbb{E}_{\pi}\!\left[ \sum_{t=0}^{\infty} (R_{t+1} - \bar{r}_{\pi}) \mid S_0 = s, A_0 = a, A_{1:\infty} \sim \pi \right].
\end{equation}
Crucially, optimizing the differential value function \eqref{eq_diff_value_function} provides a means to indirectly optimize the average-reward objective \eqref{eq_avg_reward}, such that optimizing the differential value function indirectly optimizes the average-reward objective \eqref{eq_avg_reward}. 

Importantly, to optimize the value function, differential algorithms need to maintain a running estimate of the average-reward. From an algorithmic perspective, this means that the algorithms start with initial estimates (or guesses) for the value function \emph{and} average-reward, then update these estimates over time, until they have learned and/or optimized these two objectives. There are different ways through which the average-reward estimate can be updated. One way is by using empirical reward samples from the environment (e.g. \citet{Adamczyk2025-cv}). Another such way is via the TD error. For instance, a tabular Differential Q-learning algorithm \citep{Wan2021-re}, which updates the average-reward estimate with the TD error, is shown below:

\begin{subequations}
\label{eq_diff_q_learning_tabular}
\begin{align}
\label{eq_diff_q_learning_tabular_1}
& \delta_{t} = R_{t+1} - \bar{R}_{t} +  \max_{a_*}Q_{t}(S_{t+1}, a_*) - Q_{t}(S_t, A_t)\\
\label{eq_diff_q_learning_tabular_2}
& Q_{t+1}(S_t, A_t) = Q_{t}(S_t, A_t) + \alpha_{t}\delta_{t}\\
\label{eq_diff_q_learning_tabular_3}
& Q_{t+1}(s, a) = Q_{t}(s, a), \quad \forall (s,a) \neq (S_t, A_t)\\
\label{eq_diff_q_learning_tabular_4}
& \bar{R}_{t+1} = \bar{R}_{t} + \eta\alpha_{t}\delta_{t}\, ,
\end{align}
\end{subequations}
where \(\bar{R}_{t}\) denotes the average-reward estimate and \(\eta > 0\) denotes a step size parameter.

As per Lemma \ref{lemma_1}, the choice of TD error by which to update the average-reward estimate does not matter in the tabular case. However, we will now show that in deep RL settings, the choice of explicit vs. implicit TD error for the average-reward update (i.e., Equation \eqref{eq_diff_q_learning_tabular_4}) can have a profound effect. In particular, consider Figure \ref{fig_estimates}, which shows the average-reward estimate when using a deep differential Q-learning algorithm with explicit vs. implicit TD error average-reward updates in the Breakout environment. As shown in the figure, using the explicit TD error can cause the average-reward estimate to diverge. Conversely, using the implicit TD error mitigates this instability.

Indeed, as per Figure \ref{fig_estimates}, the choice of TD error has a drastic effect on TD-based average-reward updates in deep RL settings. But why? A closer examination of the convergence proof for the \emph{tabular} Differential Q-learning algorithm \citep{Wan2021-re} reveals some insight. In particular, the convergence proof utilized a TD error representation that could be interpreted as an ``\(n\)-step variant'' of the implicit TD error (see Equation B.12 of \citet{Wan2021-re}). That is, the average-reward estimate, \(\bar{R}_t\), was shown to be a reliable estimate of the actual average-reward, \(\bar{r}_\pi\), when updated with the TD error, but the specific interpretation of the TD error used to make that claim was that of the implicit TD error. 

Accordingly, if the explicit TD error is used to update the average-reward estimate in deep RL settings, the divergence between the two TD errors degrades the quality of the average-reward estimate, and introduces instability and nonstationarity into the explicit TD error target. More formally, let \(\epsilon_t \doteq \delta^e_t - \delta^i_t\) denote the difference between the two TD errors at time \(t\), such that if at time \(t\) the average-reward estimate is updated with the explicit TD error, we have that:
\begin{subequations}
\label{eq_ar_update}
\begin{align}
\bar{R}_{t+1} &= \bar{R}_{t} + \eta\alpha_{t}\delta^e_{t}\\
& = \bar{R}_{t} + \eta\alpha_{t}(\delta^i_{t} + \epsilon_t)\\
& = \bar{R}_{t} + \eta\alpha_{t}\delta^i_{t} \textcolor{red}{\; + \; \eta\alpha_{t}\epsilon_t}\\
& \doteq \bar{R}^i_{t+1} \textcolor{red}{\; + \; \eta\alpha_{t}\epsilon_t},
\end{align}
\end{subequations}
where the \(\bar{R}^i_{t+1} \doteq \bar{R}_{t} + \eta\alpha_{t}\delta^i_{t}\) term constitutes the average-reward update that would have been made if the estimate was updated via the implicit TD error. Hence, if there is large divergence between the two TD errors, as we have seen is the case in increasingly deep RL settings, then the otherwise-reliable estimate of the average-reward is degraded by this \(\textcolor{red}{\eta\alpha_{t}\epsilon_t}\) term. Moreover, since the average-reward estimate is used to compute the explicit TD error target, this divergence introduces additional instability and nonstationarity into the agent's learning. That is, we have that:
\begin{subequations}
\label{eq_diff_target}
\begin{align}
\label{eq_diff_target_1}
\delta^e_t
&\doteq \frac{1}{B}\sum_{j=1}^{B}\left[R_{j,t+1} - \bar{R}_t + \max_{a_*}Q_{\boldsymbol{w}_t}(S_{j, t+1},a_*) - Q_{\boldsymbol{w}_t}(S_{j,t}, A_{j,t})\right]
\\
\label{eq_diff_target_2}
&= \frac{1}{B}\sum_{j=1}^{B}\left[R_{j,t+1} - \bar{R}^i_{t} \textcolor{red}{\; - \; \eta\alpha_{t-1}\epsilon_{t-1}} + \max_{a_*}Q_{\boldsymbol{w}_t}(S_{j, t+1},a_*) - Q_{\boldsymbol{w}_t}(S_{j,t}, A_{j,t})\right],
\end{align}
\end{subequations}
such that the \(\textcolor{red}{\eta\alpha_{t-1}\epsilon_{t-1}}\) term creates a source of instability and nonstationarity when learning not just the average-reward, but also the value function. Critically, if this degraded explicit TD error is then used to update the estimate of the average-reward in subsequent time steps, then the error in the average-reward estimate will compound, such that after \(n\) updates:
\begin{equation}
\label{eq_ar_update_compound}
\bar{R}_{n} = \bar{R}_{0} + \eta\sum_{t=0}^{n-1}\alpha_{t}\delta^i_{t} \textcolor{red}{\; + \; \eta\sum_{t=0}^{n-1}\alpha_{t}\epsilon_t}.
\end{equation}

\newpage

This compounded error term can hence be viewed as one of the sources of instability in the average-reward estimates displayed in Figure \ref{fig_estimates} for deep differential algorithms that utilize the explicit TD error to update the average-reward estimate. To the best of our knowledge, this result constitutes the first theoretically driven explanation for the instability of deep differential algorithms with TD error-based average-reward updates. We emphasize that this may only be one of many possible sources of instability in deep differential algorithms, and that further theoretical work may be required to explore other sources of instability.

Importantly, we note that despite the \(n\)-step implicit TD error interpretation being used in the convergence proof for tabular Differential Q-learning \citep{Wan2021-re}, to the best of our knowledge, our work is the first to propose updating the average-reward estimate by using the implicit TD error (the algorithms proposed in \citet{Wan2021-re} use the explicit TD error). Consequently, our work is only the second to provide a deep differential Q-learning algorithm with a stable TD error-based average-reward update. The first was the baseline used in \citet{Rojas2026-dv}, which used the smallest-magnitude explicit TD error in a batch, rather than the batch-average explicit TD error. However, this smallest-magnitude update is not theoretically motivated, and appears to be more of a heuristic than a principled approach. In Section \ref{experiments}, we compare our proposed differential Q-learning algorithm with implicit TD error updates (see Appendix \ref{appendix_algs} for the formal pseudocode) with the smallest-magnitude baseline used in \citet{Rojas2026-dv}.

\subsection{Reward Centering Algorithms and the TD Error}
\label{centering}
In reward centering algorithms \citep{Naik2024-km}, the agent aims to find a policy that optimizes the discounted value function (i.e., Equation \eqref{eq_disc_value}), such that all the one-step rewards are ``centered'' by subtracting their average. For example, in the context of Q-learning, we have that:
\begin{equation}
\label{eq_centering_value}
Q_{\pi}(s, a) \doteq \mathbb{E}_{\pi}\!\left[ \sum_{t=0}^{T} \gamma^{t} (R_{t+1} - \bar{r}_{\pi}) \mid S_0 = s, A_0 = a, A_{1:T} \sim \pi \right].
\end{equation}
Like the differential algorithms discussed in Section \ref{differential}, the average-reward, \(\bar{r}_{\pi}\), needs to be estimated in reward centering algorithms. There are different ways through which this estimate can be updated. In \emph{value-based reward centering}, the estimate is updated via the TD error. Crucially, like the differential methods discussed in Section \ref{differential}, the convergence proof for the average-reward estimate in tabular value-based reward centering algorithms utilizes the implicit TD error interpretation (see Appendix B in \citet{Naik2024-km}). Accordingly, one can show, using an identical process to the one described in Section \ref{differential} for differential algorithms, that value-based reward centering algorithms can be unstable when utilizing batch-averaged explicit TD error average-reward updates in increasingly deep RL settings, and that utilizing batch-averaged implicit TD error updates mitigates this instability. In Section \ref{experiments} we compare the empirical performance when using explicit vs. implicit batch-averaged TD error updates in reward centering algorithms.

\subsection{A2C Algorithms and the TD Error}
\label{a2c}
In Advantage Actor-Critic (A2C) algorithms \citep{Mnih2016-tt}, the agent aims to optimize a parameterized policy (the actor) alongside a parameterized state-value function (the critic, \(V_{\boldsymbol{w}_t}(S_{t})\)). The critic is trained to approximate the value function, while the actor is trained using an estimate of the \emph{advantage function}, which measures, at a given state, how much better taking a specific action is relative to the average performance under the current policy at that state. In particular, A2C algorithms use the TD error as a proxy for the advantage. That is, the actor is trained by minimizing a loss function in which the TD error acts as an estimate of the advantage, such that:
\begin{equation}
\label{eq_a2c_loss}
\mathcal{L}_{\text{actor}}(\boldsymbol{u}_t) \doteq - \log \pi_{\boldsymbol{u}_t}(A_t \mid S_t)\,\delta_t \, .
\end{equation}

where \(\pi_{\boldsymbol{u}_t}\) denotes the parameterized policy and \(\boldsymbol{u}_t\) denotes the policy parameters at time \(t\).

A2C algorithms that utilize the explicit TD error (in this case, in the actor loss) have been shown to yield good performance in deep RL settings \citep{Mnih2016-tt}, and do not suffer from the same instability issues as the deep differential and (value-based) reward centering algorithms described in Sections \ref{differential} and \ref{centering}, respectively. As such, we will focus our exploration of A2C algorithms in terms of comparing the empirical performance when using the explicit vs. implicit TD errors as estimates of the advantage. This empirical comparison is done in Section \ref{experiments}.

\newpage

In this regard, one theoretical consideration remains: \textit{does the implicit TD error provide a viable estimate of the advantage function?} In Proposition \ref{proposition_1}, we argue that the implicit TD error indeed provides a viable estimate of the advantage function:
\vspace{4pt}
\begin{proposition}
\label{proposition_1}
The implicit TD error provides a viable estimate of the advantage function.
\end{proposition}
\begin{proof}
Consider the A2C implicit TD error, \(\delta^{i}_t \doteq \frac{1}{\alpha_t}\left(V_{\boldsymbol{w}_{t+1}}(S_{t}) - V_{\boldsymbol{w}_t}(S_{t})\right)\), which measures how much the agent's state-value estimate, \(V_{\boldsymbol{w}_t}(S_{t})\), changes as a result of taking action \(A_t\). Since a positive revision indicates that the observed outcome was better than the agent expected from state \(S_t\), and a negative revision indicates the opposite, it follows that the implicit TD error is directionally consistent with, and thus constitutes a viable estimate of, the advantage function.
\end{proof}

\subsection{The Empirical Implications of Using the Explicit vs. Implicit TD Errors}
\label{experiments}
In this section, we perform an empirical evaluation of how using the explicit TD error vs. the implicit TD error affects the performance of deep differential, reward centering, and A2C algorithms. In particular, for each class of algorithm, we compared how the algorithm performs in Atari \citep{Bellemare2013-ef} or MuJoCo \citep{Todorov2012-ku} environments when using the explicit vs. implicit TD error to compute other quantities (i.e., the average-reward or advantage estimates). We tested the differential algorithms in the Atari \emph{Breakout} environment, the reward centering algorithms in the Atari \emph{Pong} environment, and the A2C algorithms in the MuJoCo \emph{HalfCheetah} environment.

Figure \ref{fig_results} shows the results of this comparison. In particular, Figure \ref{fig_results}\textcolor{mylightblue}{a)} shows that for differential algorithms, using the implicit TD error to update the average-reward estimate yields superior performance in the early stages of learning, precisely when the divergence between the two TD errors is largest (see Figure \ref{fig_td_diff}). We can also see (in Figure \ref{fig_results}\textcolor{mylightblue}{a)}) that using the implicit TD error to update the average-reward estimate results in more stable performance than when using the smallest-magnitude explicit TD error average-reward update from \citet{Rojas2026-dv}. In Figure \ref{fig_results}\textcolor{mylightblue}{b)}, we can see that, for value-based reward centering algorithms, using the implicit TD error to update the average-reward estimate yields superior performance. We can also see that, in this case, using the explicit TD error to update the average-reward estimate results in stable performance. This is consistent with the findings from \citet{Naik2024-km}, which suggest that, in some cases, using the explicit TD error to update the average-reward estimate can be a viable option given sufficient step size tuning. Finally, Figure \ref{fig_results}\textcolor{mylightblue}{c)} shows that for A2C algorithms, using the implicit TD error as the advantage estimate yields slightly superior performance than when using the explicit TD error as the advantage estimate. The full set of experimental details and results can be found in Appendix \ref{appendix_experiments}.

\begin{figure}[htbp]
\centerline{\includegraphics[scale=0.53]{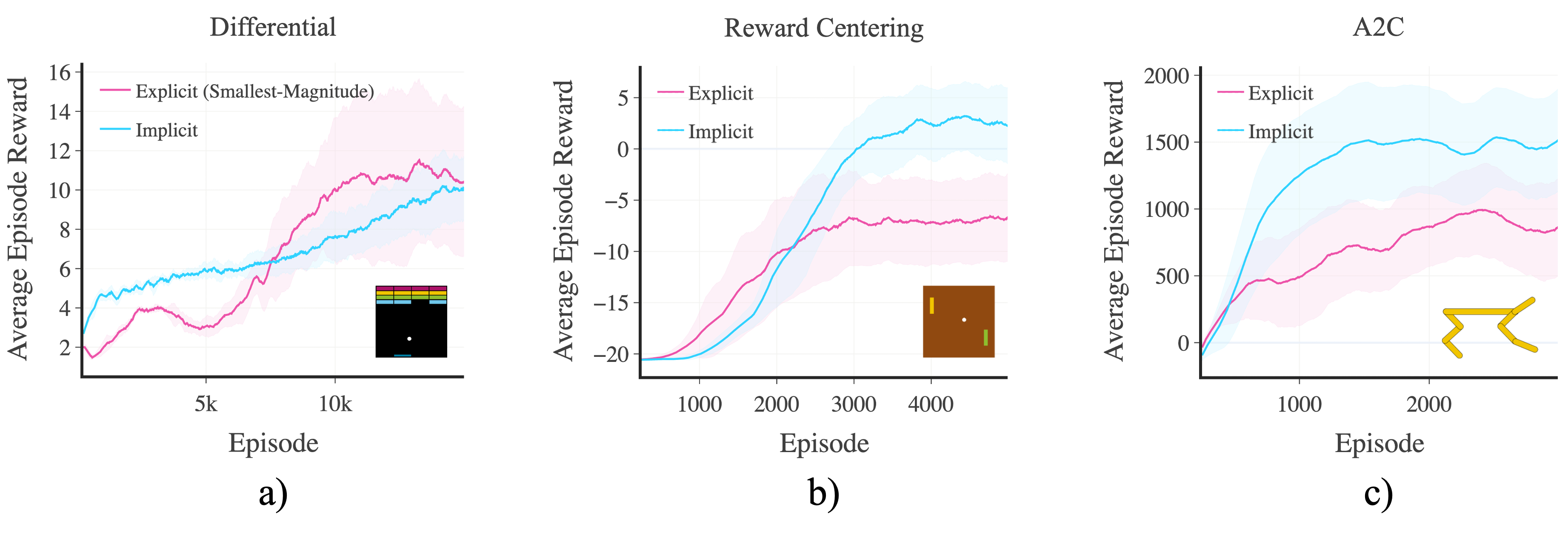}}
\caption{Rolling total reward per episode when using: \textbf{a)} a differential Q-learning algorithm in the Breakout environment, \textbf{b)} a Q-learning algorithm with value-based reward centering in the Pong environment, and \textbf{c)} an A2C algorithm in the HalfCheetah environment. Each plot shows the performance of each algorithm when using the explicit vs. implicit TD error. A solid line denotes the mean total reward per episode, and a shaded region denotes a 95\% confidence interval over 8 runs.}
\label{fig_results}
\end{figure}

\section{Discussion, Limitations, and Future Work}
In this work, we performed a formal exploration and characterization of the differences between two seemingly-equivalent interpretations of the TD error: the TD error as the difference between temporally successive predictions (the \emph{implicit} TD error), and the TD error as the difference between a bootstrapped target and a prediction (the \emph{explicit} TD error). We showed that although these two interpretations are equivalent in tabular settings, increasingly nonlinear deep RL architectures can cause these two interpretations of the TD error to yield increasingly different numerical values. We then leveraged this insight to show how choosing one interpretation over the other can affect the performance of deep RL algorithms that utilize the TD error to compute other quantities. 

Most notably, we showed with differential (i.e., average-reward) algorithms that using the interpretation of the TD error as the difference between a bootstrapped target and a prediction can result in unstable average-reward estimates in deep RL settings. Conversely, we showed how using the interpretation of the TD error as the difference between temporally successive predictions mitigates this instability. From an empirical perspective, we also showed how the specific choice of TD error used in deep RL algorithms can result in more stable and sometimes superior performance, not just with differential algorithms, but also with reward centering \citep{Naik2024-km} and A2C \citep{Mnih2016-tt} algorithms. Altogether, our results show that the default interpretation of the TD error as the difference between a bootstrapped target and a prediction does not always hold in deep RL settings. In other words, \emph{the choice of implicit vs. explicit TD error matters in deep RL settings}.

In terms of limitations and future work, we note that our analytical investigation was limited to showing the existence of a difference between the two TD errors in increasingly nonlinear deep RL settings. Although we formally quantified this difference in linear function approximation settings, quantifying (or bounding) this difference in deep RL settings remains an open research question. In this regard, the type of nonlinear activations used, the number and type of layers used, the size of batch updates, and the type of optimizer used (e.g. stochastic gradient descent vs. Adam \citep{Kingma2015-hm}) are all possible sources of divergence between the two TD errors in deep RL settings. 

More broadly, this work opens the door to several interesting research directions that could be pursued in future work. Namely, now that we have formalized the notion of the implicit TD error in deep RL settings, can we design deep RL algorithms that directly optimize this objective rather than, or alongside, the explicit TD error? Are there other ways in which the implicit TD error can be incorporated into deep RL algorithms to facilitate learning? Indeed, the implications of the implicit TD error in deep RL settings presents a compelling direction for future research, one that suggests that the story of the TD error as it pertains to RL has only just begun.

\bibliography{references}

\begin{thebibliography}{15}
\providecommand{\natexlab}[1]{#1}
\providecommand{\url}[1]{\texttt{#1}}
\expandafter\ifx\csname urlstyle\endcsname\relax
  \providecommand{\doi}[1]{doi: #1}\else
  \providecommand{\doi}{doi: \begingroup \urlstyle{rm}\Url}\fi

\bibitem[Adamczyk et~al.(2025)Adamczyk, Makarenko, Tiomkin, and Kulkarni]{Adamczyk2025-cv}
Jacob Adamczyk, Volodymyr Makarenko, Stas Tiomkin, and Rahul~V Kulkarni.
\newblock Average-reward soft actor-critic.
\newblock \emph{Reinforcement Learning Journal}, 2025.

\bibitem[Bellemare et~al.(2013)Bellemare, Naddaf, Veness, and Bowling]{Bellemare2013-ef}
Marc~G Bellemare, Yavar Naddaf, J~Veness, and Michael Bowling.
\newblock The arcade learning environment: An evaluation platform for general agents.
\newblock \emph{J. Artif. Intell. Res.}, pages 253--279, 2013.

\bibitem[Haarnoja et~al.(2018)Haarnoja, Zhou, Abbeel, and Levine]{Haarnoja2018-xd}
Tuomas Haarnoja, Aurick Zhou, Pieter Abbeel, and Sergey Levine.
\newblock Soft actor-critic: Off-policy maximum entropy deep reinforcement learning with a stochastic actor.
\newblock In \emph{Proceedings of the 35th International Conference on Machine Learning}, 2018.

\bibitem[Kingma and Ba(2015)]{Kingma2015-hm}
Diederik~P Kingma and Jimmy Ba.
\newblock Adam: A method for stochastic optimization.
\newblock In \emph{International Conference on Learning Representations (ICLR)}, 2015.

\bibitem[Mnih et~al.(2015)Mnih, Kavukcuoglu, Silver, Rusu, Veness, Bellemare, Graves, Riedmiller, Fidjeland, Ostrovski, Petersen, Beattie, Sadik, Antonoglou, King, Kumaran, Wierstra, Legg, and Hassabis]{Mnih2015-un}
Volodymyr Mnih, Koray Kavukcuoglu, David Silver, Andrei~A Rusu, Joel Veness, Marc~G Bellemare, Alex Graves, Martin Riedmiller, Andreas~K Fidjeland, Georg Ostrovski, Stig Petersen, Charles Beattie, Amir Sadik, Ioannis Antonoglou, Helen King, Dharshan Kumaran, Daan Wierstra, Shane Legg, and Demis Hassabis.
\newblock Human-level control through deep reinforcement learning.
\newblock \emph{Nature}, 518\penalty0 (7540):\penalty0 529--533, February 2015.

\bibitem[Mnih et~al.(2016)Mnih, Badia, Mirza, Graves, Lillicrap, Harley, Silver, and Kavukcuoglu]{Mnih2016-tt}
Volodymyr Mnih, Adrià~Puigdomènech Badia, Mehdi Mirza, Alex Graves, Timothy~P Lillicrap, Tim Harley, David Silver, and Koray Kavukcuoglu.
\newblock Asynchronous methods for deep reinforcement learning.
\newblock In \emph{Proceedings of the 33rd International Conference on Machine Learning}, 2016.

\bibitem[Naik et~al.(2024)Naik, Wan, Tomar, and Sutton]{Naik2024-km}
Abhishek Naik, Yi~Wan, Manan Tomar, and Richard~S Sutton.
\newblock Reward centering.
\newblock \emph{Reinforcement Learning Journal}, 4:\penalty0 1995–2016, 2024.

\bibitem[Puterman(1994)]{Puterman1994-dq}
Martin~L Puterman.
\newblock \emph{Markov Decision Processes: Discrete Stochastic Dynamic Programming}.
\newblock John Wiley \& Sons, 1994.

\bibitem[Rojas and Lee(2026)]{Rojas2026-dv}
Juan~Sebastian Rojas and Chi-Guhn Lee.
\newblock A differential perspective on distributional reinforcement learning.
\newblock In \emph{Proceedings of the AAAI Conference on Artificial Intelligence}, 2026.

\bibitem[Sutton(1988)]{Sutton1988-vs}
Richard~S Sutton.
\newblock Learning to predict by the methods of temporal differences.
\newblock \emph{Mach. Learn.}, 3:\penalty0 944, 1988.

\bibitem[Sutton and Barto(2018)]{Sutton2018-eh}
Richard~S Sutton and Andrew~G Barto.
\newblock \emph{Reinforcement Learning: An Introduction, {2nd} edition}.
\newblock MIT Press, November 2018.

\bibitem[Tesauro(1992)]{Tesauro1992-ul}
Gerald Tesauro.
\newblock Practical issues in temporal difference learning.
\newblock \emph{Mach. Learn.}, 8\penalty0 (3-4):\penalty0 257--277, May 1992.

\bibitem[Todorov et~al.(2012)Todorov, Erez, and Tassa]{Todorov2012-ku}
Emanuel Todorov, Tom Erez, and Yuval Tassa.
\newblock {MuJoCo}: A physics engine for model-based control.
\newblock In \emph{2012 IEEE/RSJ International Conference on Intelligent Robots and Systems}, pages 5026--5033. IEEE, October 2012.

\bibitem[Wan et~al.(2021)Wan, Naik, and Sutton]{Wan2021-re}
Yi~Wan, Abhishek Naik, and Richard~S Sutton.
\newblock Learning and planning in average-reward markov decision processes.
\newblock In \emph{Proceedings of the 38th International Conference on Machine Learning}, 2021.

\bibitem[Watkins and Dayan(1992)]{Watkins1992-nq}
Christopher~J Watkins and Peter Dayan.
\newblock {Q}-learning.
\newblock \emph{Mach. Learn.}, 8:\penalty0 279--292, 1992.

\end{thebibliography}


\newpage
\appendix
\numberwithin{equation}{section}
\numberwithin{figure}{section}
\numberwithin{theorem}{subsection}

\section{Deep RL Algorithms with Implicit TD Error Updates}
\label{appendix_algs}

In this appendix, we provide the pseudocode for our proposed deep RL algorithms with implicit TD error updates. As is common practice in deep RL settings, we adopt the smooth L1 loss, \(L^{\lambda}(x)\), \citep{Mnih2015-un} as the critic loss in our algorithms:
\begin{equation}
\label{eq_smoothl1}
L^{\lambda}(x) = 
\begin{cases}
\frac{1}{2\lambda}x^2, & \text{if } |x| \leq \lambda \\
|x| - \frac{1}{2}\lambda, & \text{if } |x| > \lambda \text{.}
\end{cases}
\end{equation}

\begin{algorithm}
   \caption{Differential DQN with Implicit TD Error Average-Reward Updates}
   \label{alg_1}
\begin{algorithmic}
    \STATE {\bfseries Input:} a differentiable state-action value function parameterization: \(\hat{q}(s, a, \boldsymbol{w})\) (with target network \(\hat{q}_{_T}(s, a, \boldsymbol{w_{_T}})\)), the policy \(\pi\) to be used (e.g., \(\varepsilon\)-greedy)
    \STATE {\bfseries Algorithm parameters:} step size parameters \(\{\alpha\), \(\eta\}\), loss parameter \(\lambda\)
    \STATE Initialize state-action value weights \(\boldsymbol{w}, \boldsymbol{w_{_T}} \in \mathbb{R}^{d}\) arbitrarily (e.g. to \(\boldsymbol{0}\))    
    \STATE Initialize \(\bar{R}\) arbitrarily (e.g. to zero)
    \STATE Obtain initial \(S\)
    \WHILE{still time to train}
        \STATE \(A \leftarrow\) action given by \(\pi\) for \(S\)
        \STATE Take action \(A\), observe \(R, S'\)
        \STATE Store \((S, A, R, S')\) in replay buffer
        \IF {time to update estimates}
        \STATE Sample a minibatch of \(B\) transitions from replay buffer: \(\{(S_b, A_b, R_b, S_b')\}_{b=1}^{B}\)
        \STATE For each \(b\)-th transition: \(\delta^e_b = R_b - \bar{R} + \max_a \hat{q}_{_T}(S_b', a, \boldsymbol{w_{_T}}) - \hat{q}(S_b, A_b, \boldsymbol{w})\)
        \STATE \(\ell = \frac{1}{B}\sum_{b=1}^{B}L^{\lambda}(\delta^e_b)\) (See Equation \eqref{eq_smoothl1})
        \STATE \(\boldsymbol{w'} = \boldsymbol{w} - \alpha\frac{\partial \ell}{\partial \boldsymbol{w}}\)
        \STATE For each \(b\)-th transition: \(\delta^i_b = \frac{1}{\alpha}\left(\hat{q}(S_b, A_b, \boldsymbol{w'}) - \hat{q}(S_b, A_b, \boldsymbol{w})\right)\)
        \STATE \(\bar{R} = \bar{R} + \eta\alpha\frac{1}{B}\sum_{b=1}^{B}\delta^i_b\)
        \STATE \(\boldsymbol{w} = \boldsymbol{w'}\)
        \STATE Update \(\boldsymbol{w_{_T}}\) as needed (e.g. using Polyak averaging)
        \ENDIF
        \STATE \(S = S'\)
    \ENDWHILE
    \STATE return \(\boldsymbol{w}\)
\end{algorithmic}
\end{algorithm}

\begin{algorithm}
   \caption{Value-Based Reward Centering DQN with Implicit TD Error Average-Reward Updates}
   \label{alg_2}
\begin{algorithmic}
    \STATE {\bfseries Input:} a differentiable state-action value function parameterization: \(\hat{q}(s, a, \boldsymbol{w})\) (with target network \(\hat{q}_{_T}(s, a, \boldsymbol{w_{_T}})\)), the policy \(\pi\) to be used (e.g., \(\varepsilon\)-greedy)
    \STATE {\bfseries Algorithm parameters:} discount factor \(\gamma\), step size parameters \(\{\alpha\), \(\eta\}\), loss parameter \(\lambda\)
    \STATE Initialize state-action value weights \(\boldsymbol{w}, \boldsymbol{w_{_T}} \in \mathbb{R}^{d}\) arbitrarily (e.g. to \(\boldsymbol{0}\))    
    \STATE Initialize \(\bar{R}\) arbitrarily (e.g. to zero)
    \STATE Obtain initial \(S\)
    \WHILE{still time to train}
        \STATE \(A \leftarrow\) action given by \(\pi\) for \(S\)
        \STATE Take action \(A\), observe \(R, S'\)
        \STATE Store \((S, A, R, S')\) in replay buffer
        \IF {time to update estimates}
        \STATE Sample a minibatch of \(B\) transitions from replay buffer: \(\{(S_b, A_b, R_b, S_b')\}_{b=1}^{B}\)
        \STATE For each \(b\)-th transition: \(\delta^e_b = R_b - \bar{R} + \gamma \max_a \hat{q}_{_T}(S_b', a, \boldsymbol{w_{_T}}) - \hat{q}(S_b, A_b, \boldsymbol{w})\)
        \STATE \(\ell = \frac{1}{B}\sum_{b=1}^{B}L^{\lambda}(\delta^e_b)\) (See Equation \eqref{eq_smoothl1})
        \STATE \(\boldsymbol{w'} = \boldsymbol{w} - \alpha\frac{\partial \ell}{\partial \boldsymbol{w}}\)
        \STATE For each \(b\)-th transition: \(\delta^i_b = \frac{1}{\alpha}\left(\hat{q}(S_b, A_b, \boldsymbol{w'}) - \hat{q}(S_b, A_b, \boldsymbol{w})\right)\)
        \STATE \(\bar{R} = \bar{R} + \eta\alpha\frac{1}{B}\sum_{b=1}^{B}\delta^i_b\)
        \STATE \(\boldsymbol{w} = \boldsymbol{w'}\)
        \STATE Update \(\boldsymbol{w_{_T}}\) as needed (e.g. using Polyak averaging)
        \ENDIF
        \STATE \(S = S'\)
    \ENDWHILE
    \STATE return \(\boldsymbol{w}\)
\end{algorithmic}
\end{algorithm}

\begin{algorithm}
   \caption{A2C with Implicit TD Error Advantage Estimate}
   \label{alg_3}
\begin{algorithmic}
    \STATE {\bfseries Input:} a differentiable state-value function parameterization \(\hat{v}(s, \boldsymbol{w})\), a differentiable policy parameterization \(\pi(a \mid s, \boldsymbol{u})\)
    \STATE {\bfseries Algorithm parameters:} discount factor \(\gamma\), step size parameters \(\{\alpha\), \(\eta\}\)
    \STATE Initialize state-value weights \(\boldsymbol{w} \in \mathbb{R}^{d}\) and policy weights \(\boldsymbol{u} \in \mathbb{R}^{d'}\) (e.g. to \(\boldsymbol{0}\))
    \STATE Obtain initial \(S\)
    \WHILE{still time to train}
        \STATE \(A \sim \pi(\cdot \mid S, \boldsymbol{u})\)
        \STATE Take action \(A\), observe \(R, S'\)
        \STATE \(\delta^e = R + \gamma \hat{v}(S', \boldsymbol{w}) - \hat{v}(S, \boldsymbol{w})\)
        \STATE \(\boldsymbol{w'} = \boldsymbol{w} + \alpha\delta^e\nabla\hat{v}(S, \boldsymbol{w})\)
        \STATE \(\delta^i = \frac{1}{\alpha}\left(\hat{v}(S, \boldsymbol{w'}) - \hat{v}(S, \boldsymbol{w})\right)\)
        \STATE \(\boldsymbol{u} = \boldsymbol{u} + \eta\alpha\delta^i\nabla \ln \pi(A \mid S, \boldsymbol{u})\)
        \STATE \(\boldsymbol{w} = \boldsymbol{w'}\)
        \STATE \(S = S'\)
    \ENDWHILE
    \STATE return \(\boldsymbol{w}\) and \(\boldsymbol{u}\)
\end{algorithmic}
\end{algorithm}

\newpage

\section{Numerical Experiments}
\label{appendix_experiments}
This appendix contains details regarding the numerical experiments performed as part of this work. 

We ran three groups of experiments. In the first group of experiments (Section \ref{exp_diff}), we aimed to show, from an empirical perspective, how the explicit and implicit TD errors can yield increasingly different numerical values in increasingly nonlinear deep RL settings. In the second group of experiments (Section \ref{exp_estimates}), we aimed to show, from an empirical perspective, how the choice of explicit vs. implicit TD error for the average-reward update (i.e., Equation \eqref{eq_diff_q_learning_tabular_4}) in deep differential (i.e., average-reward) algorithms can have a profound effect on the stability of the average-reward estimate. In the third group of experiments (Section \ref{exp_performance}), we aimed to show, from an empirical perspective, how using the explicit TD error vs. the implicit TD error affects the performance of deep differential, reward centering \citep{Naik2024-km}, and A2C \citep{Mnih2016-tt} algorithms.

\subsection{TD Error Difference Experiments}
\label{exp_diff}
In this section, we discuss the TD error difference experiments performed as part of this work. Through these experiments, we aimed to show, from an empirical perspective, how the explicit and implicit TD errors can yield increasingly different numerical values in increasingly nonlinear deep RL settings.

To this end, we compared the explicit and implicit TD errors when applying Q-learning algorithms on two RL tasks. In the first task, we used a Q-learning algorithm on the well-known \emph{inverted pendulum} environment, where an agent learns how to optimally balance an inverted pendulum. Then, in the second task, we used a DQN algorithm \citep{Mnih2015-un} on the Atari \emph{Breakout} environment \citep{Bellemare2013-ef}. Both environments are displayed in Figure \ref{fig_experiments_diff}. In both experiments, we compared the two TD errors when introducing increased nonlinearity into the function approximator. 
\begin{figure}[htbp]
\centerline{\includegraphics[scale=0.6]{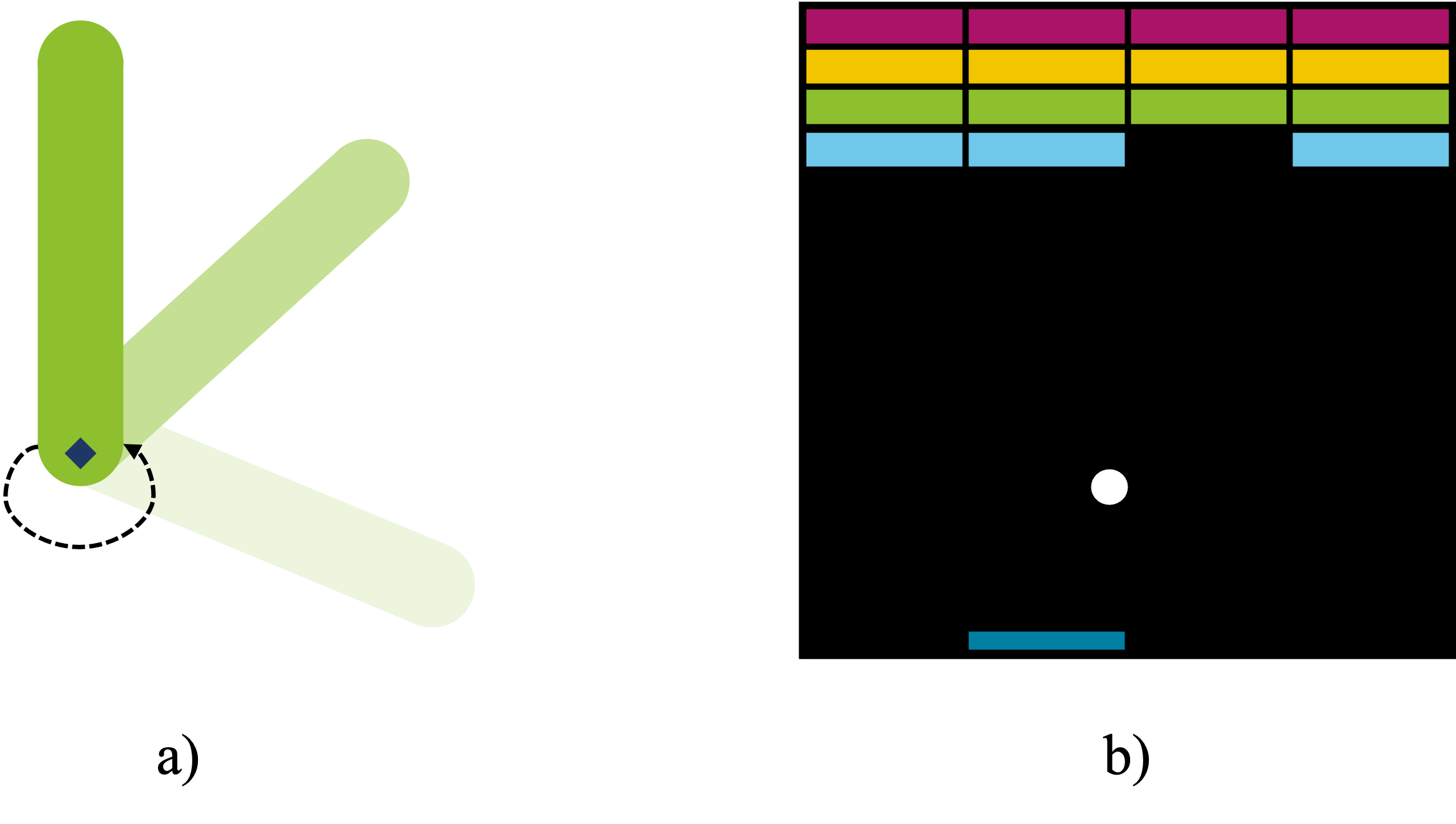}}
\caption{An illustration of the \textbf{a)} inverted pendulum, and \textbf{b)} Atari Breakout environments.}
\label{fig_experiments_diff}
\end{figure}

In the inverted pendulum experiment, we compared the two TD errors when using a linear function approximator (with no batch updates) vs. a function approximator with two hidden layers and batch updates. For the linear function approximator, we used tile coding \citep{Sutton2018-eh} with 32 tilings, each with 8 X 8 tiles. For the nonlinear function approximator, we used the following network
with \textit{hidden\_dim} = 32:

\begin{lstlisting}[language=Python, numbers=none]
import torch

class ValueNetwork(torch.nn.Module):
    def __init__(self, state_dim, action_dim, hidden_dim):
        super(ValueNetwork, self).__init__()
        self.linear_1 = torch.nn.Linear(state_dim, hidden_dim)
        self.linear_2 = torch.nn.Linear(hidden_dim, hidden_dim)
        self.linear_3 = torch.nn.Linear(hidden_dim, action_dim)


    def forward(self, x):
        x = torch.nn.functional.relu(self.linear_1(x))
        x = torch.nn.functional.relu(self.linear_2(x))
        return self.linear_3(x)
\end{lstlisting}

For this nonlinear function approximator, we used a replay buffer with a minimum size of 1e2, a maximum size of 1e5, and a sample batch size of 32. We used a target network, and updated it using Polyak averaging with a fixed step size of 0.005. For both approximators, we used stochastic gradient descent updates on the standard mean square value loss (i.e., Equation \eqref{eq_dqn_1}), a value function step size of 2e-4, a discount factor of 0.99, and an \(\varepsilon\)-greedy policy with \(\varepsilon=\) 0.1.

In the Breakout experiment, we compared the two TD errors when using the following ``small'' and ``large'' networks, both of which consist of three convolutional layers and two fully connected layers:

``Small'' network:
\begin{lstlisting}[language=Python, numbers=none]
import torch

class ValueNetwork(torch.nn.Module):
    def __init__(self, state_dim, action_dim):
        super(ValueNetwork, self).__init__()
        self.conv1 = torch.nn.Conv2d(state_dim, 16, kernel_size=8, stride=4)
        self.conv2 = torch.nn.Conv2d(16, 32, kernel_size=4, stride=2)
        self.conv3 = torch.nn.Conv2d(32, 32, kernel_size=3, stride=1)
        self.fc1 = torch.nn.Linear(32 * 7 * 7, 256)
        self.fc2 = torch.nn.Linear(256, action_dim)
        
    def forward(self, x):
        x = torch.nn.functional.relu(self.conv1(x))
        x = torch.nn.functional.relu(self.conv2(x))
        x = torch.nn.functional.relu(self.conv3(x))
        x = x.view(x.size(0), -1)
        x = torch.nn.functional.relu(self.fc1(x))
        return self.fc2(x)
\end{lstlisting}

``Large'' network:
\begin{lstlisting}[language=Python, numbers=none]
import torch

class ValueNetwork(torch.nn.Module):
    def __init__(self, state_dim, action_dim):
        super(ValueNetwork, self).__init__()
        self.conv1 = torch.nn.Conv2d(state_dim, 64, kernel_size=8, stride=4)
        self.conv2 = torch.nn.Conv2d(64, 128, kernel_size=4, stride=2)
        self.conv3 = torch.nn.Conv2d(128, 128, kernel_size=3, stride=1)
        self.fc1 = torch.nn.Linear(128 * 7 * 7, 1024)
        self.fc2 = torch.nn.Linear(1024, action_dim)

    def forward(self, x):
        x = torch.nn.functional.relu(self.conv1(x))
        x = torch.nn.functional.relu(self.conv2(x))
        x = torch.nn.functional.relu(self.conv3(x))
        x = x.view(x.size(0), -1)
        x = torch.nn.functional.relu(self.fc1(x))
        return self.fc2(x)
\end{lstlisting}

These networks vary by the number of learnable parameters, such that the ``small'' network consists of 0.4M parameters, and the ``large'' network consists of 6.7M parameters. The calculation of the number of parameters is shown for both networks in Tables \ref{table_1} and \ref{table_2}, respectively:

\begin{table}[h]
\centering
\begin{tabular}{|c|c|c|c|c|}
\hline
\textbf{Layer} & \textbf{Input Shape} & \textbf{Output Shape} & \textbf{Calculation Formula} & \textbf{Parameters} \\
\hline
conv1 & \(4 \times 84 \times 84\) & \(16 \times 20 \times 20\) & \((16 \times 4 \times 8 \times 8) + 16\) & 4,112 \\
\hline
conv2 & \(16 \times 20 \times 20\) & \(32 \times 9 \times 9\) & \((32 \times 16 \times 4 \times 4) + 32\) & 8,224 \\
\hline
conv3 & \(32 \times 9 \times 9\) & \(32 \times 7 \times 7\) & \((32 \times 32 \times 3 \times 3) + 32\) & 9,248 \\
\hline
fc1 & 1,568 & 256 & \((256 \times 1,568) + 256\) & 401,664 \\
\hline
fc2 & 256 & 4 & \((4 \times 256) + 4\) & 1,028 \\
\hline
\textbf{Total} & & & & \textbf{424,276} \\
\hline
\end{tabular}
\caption{Network Parameter Calculation for ``Small'' Network}
\label{table_1}
\end{table}

\begin{table}[h]
\centering
\begin{tabular}{|c|c|c|c|c|}
\hline
\textbf{Layer} & \textbf{Input Shape} & \textbf{Output Shape} & \textbf{Calculation Formula} & \textbf{Parameters} \\
\hline
conv1 & \(4 \times 84 \times 84\) & \(64 \times 20 \times 20\) & \((64 \times 4 \times 8 \times 8) + 64\) & 16,448 \\
\hline
conv2 & \(64 \times 20 \times 20\) & \(128 \times 9 \times 9\) & \((128 \times 64 \times 4 \times 4) + 128\) & 131,200 \\
\hline
conv3 & \(128 \times 9 \times 9\) & \(128 \times 7 \times 7\) & \((128 \times 128 \times 3 \times 3) + 128\) & 147,584 \\
\hline
fc1 & 6,272 & 1,024 & \((1,024 \times 6,272) + 1,024\) & 6,423,552 \\
\hline
fc2 & 1,024 & 4 & \((4 \times 1,024) + 4\) & 4,100 \\
\hline
\textbf{Total} & & & & \textbf{6,722,884} \\
\hline
\end{tabular}
\caption{Network Parameter Calculation for ``Large'' Network}
\label{table_2}
\end{table}

For both networks, we used the Adam optimizer \citep{Kingma2015-hm} with the default hyperparameters on the smooth L1 loss (Equation \eqref{eq_smoothl1}) with a \(\lambda\) of 1.0. We used a replay buffer with a minimum size of 1e2, a maximum size of 1e5, and a sample batch size of 32. We used a target network, and updated it using Polyak averaging with a fixed step size of 0.005. We used a value function step size of 2e-4, a discount factor of 0.99, and an \(\varepsilon\)-greedy policy with \(\varepsilon=\) 0.1.

\subsection{Average-Reward Estimate Experiments}
\label{exp_estimates}
In this section, we discuss the average-reward estimate experiments performed as part of this work. Through these experiments, we aimed to show, from an empirical perspective, how the choice of explicit vs. implicit TD error for the average-reward update in deep differential (i.e., average-reward) algorithms can have a profound effect on the stability of the average-reward estimate.

To this end, we compared a deep differential Q-learning algorithm that utilizes \emph{implicit} TD error average-reward updates (Algorithm \ref{alg_1}) against a deep differential Q-learning algorithm that utilizes \emph{explicit} TD error average-reward updates (Algorithm \ref{alg_4}). 

For these experiments, we used the following network architecture:

\begin{lstlisting}[language=Python, numbers=none]
import torch

class ValueNetwork(torch.nn.Module):
    def __init__(self, state_dim, action_dim):
        super(ValueNetwork, self).__init__()
        self.conv1 = torch.nn.Conv2d(state_dim, 32, kernel_size=8, stride=4)
        self.conv2 = torch.nn.Conv2d(32, 64, kernel_size=4, stride=2)
        self.conv3 = torch.nn.Conv2d(64, 64, kernel_size=3, stride=1)
        self.fc1 = torch.nn.Linear(64 * 7 * 7, 512)
        self.fc2 = torch.nn.Linear(512, action_dim)

    def forward(self, x):
        x = torch.nn.functional.relu(self.conv1(x))
        x = torch.nn.functional.relu(self.conv2(x))
        x = torch.nn.functional.relu(self.conv3(x))
        x = x.view(x.size(0), -1)
        x = torch.nn.functional.relu(self.fc1(x))
        return self.fc2(x)
\end{lstlisting}

For both algorithms, we used the Adam optimizer with the default hyperparameters on the smooth L1 loss (Equation \eqref{eq_smoothl1}) with a \(\lambda\) of 1.0. We used a replay buffer with a minimum size of 1e2, a maximum size of 1e5, and a sample batch size of 32. We used a target network, and updated it using Polyak averaging with a fixed step size of 0.005. We used a value function step size, \(\alpha\), of 2e-5, an average-reward step size, \(\eta\alpha\), with an \(\eta\) of 1.0, and an \(\varepsilon\)-greedy policy with \(\varepsilon=\) 0.1. We ran the experiment for different initial average-reward estimate guesses: \(\bar{R}_0 \in \{\text{-0.25, 0.0, 0.25}\}\).

As shown in Figure \ref{fig_estimates}, utilizing the explicit TD error to update the average-reward estimate can cause the average-reward estimate to diverge. Conversely, the figure shows that using the implicit TD error mitigates this instability. To complement these results, Figure \ref{fig_estimates_results} shows the effect that this instability has on the overall performance of the algorithm. As shown in the figure, the instability caused by the explicit TD error average-reward updates results in significantly lower performance than when using implicit TD error average-reward updates.
\begin{figure}[htbp]
\centerline{\includegraphics[scale=0.53]{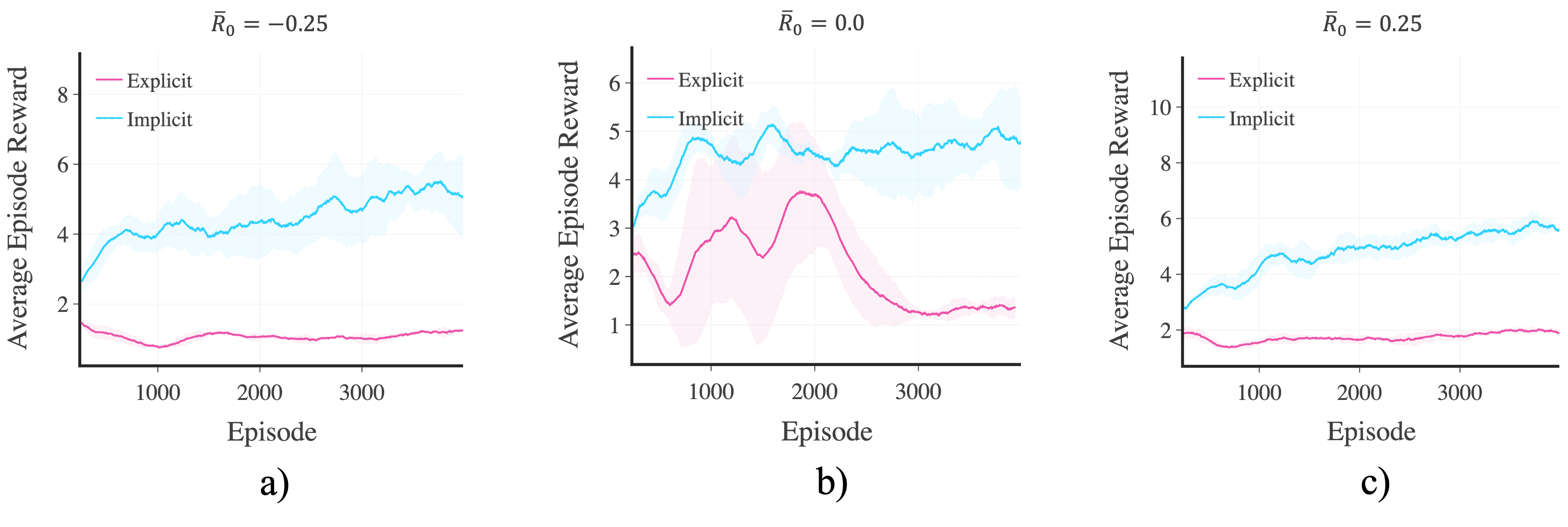}}
\caption{Rolling total reward per episode when using a deep differential Q-learning algorithm in the Breakout environment with an initial average-reward guess of: \textbf{a)} -0.25, \textbf{b)} 0.0, and \textbf{c)} 0.25. Each plot shows the performance when using (batch-averaged) explicit vs. implicit TD error average-reward updates. A solid line denotes the mean total reward per episode, and a shaded region denotes a 95\% confidence interval over 4 runs.}
\label{fig_estimates_results}
\end{figure}

\subsection{Performance Experiments}
\label{exp_performance}
In this section, we discuss the performance experiments performed as part of this work. Through these experiments, we aimed to show, from an empirical perspective, how using the explicit TD error vs. the implicit TD error affects the performance of deep differential, reward centering \citep{Naik2024-km}, and A2C \citep{Mnih2016-tt} algorithms. 

\begin{figure}[htbp]
\centerline{\includegraphics[scale=0.53]{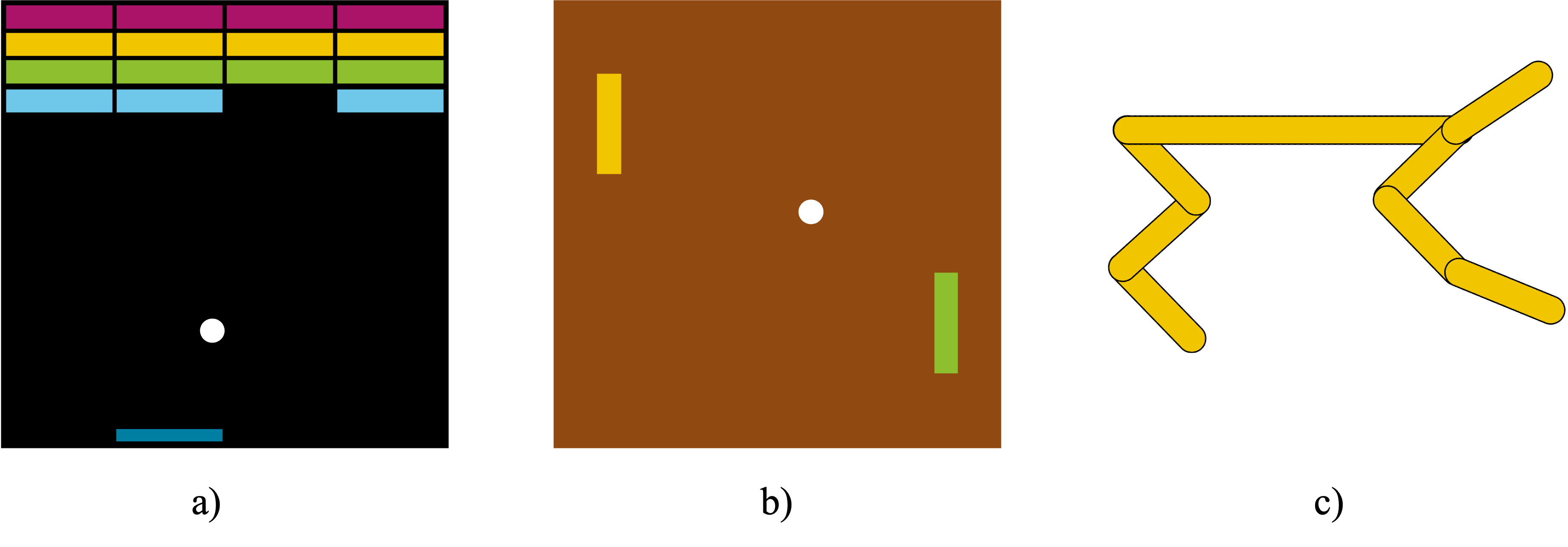}}
\caption{Illustration of the \textbf{a)} Atari Breakout, \textbf{b)} Atari Pong, and \textbf{c)} MuJoCo HalfCheetah tasks.}
\label{fig_experiments}
\end{figure}

In particular, for each class of algorithm, we compared how the algorithm performs in Atari \citep{Bellemare2013-ef} or MuJoCo \citep{Todorov2012-ku} environments when using the explicit vs. implicit TD error to compute other quantities (i.e., the average-reward or advantage estimates). We tested the differential algorithms in the Atari \emph{Breakout} environment, the reward centering algorithms in the Atari \emph{Pong} environment, and the A2C algorithms in the MuJoCo \emph{HalfCheetah} environment. These tasks are shown in Figure \ref{fig_experiments}.\\

\textbf{Breakout Experiment:}

In the Breakout experiment, we compared a deep differential Q-learning algorithm that utilizes implicit TD error average-reward updates (Algorithm \ref{alg_1}) against the deep differential Q-learning algorithm from \citet{Rojas2026-dv} that utilizes smallest-magnitude explicit TD error average-reward updates (Algorithm \ref{alg_5}).

For both algorithms, we used the Adam optimizer with the default hyperparameters on the smooth L1 loss (Equation \eqref{eq_smoothl1}) with a \(\lambda\) of 1.0. We used a replay buffer with a minimum size of 1e2, a maximum size of 1e5, and a sample batch size of 32. We used a target network, and updated it using Polyak averaging with a fixed step size of 0.005. We used an \(\varepsilon\)-greedy policy with \(\varepsilon=\) 0.1. We used an initial average-reward estimate guess of 0.25. We used the value network shown in Section \ref{exp_estimates}.

For the differential algorithm with implicit TD error average-reward updates, we tested every combination of the value function step size, \(\alpha\in\{\text{2e-6, 2e-5, 2e-4, 2e-3}\}\), with the average-reward step size, \(\eta\alpha\), where \(\eta\in\{\text{1e-2, 1e-1, 1.0, 10.0, 100.0}\}\), for a total of 20 unique combinations. Each combination was run 3 times using different random seeds for 1M steps each, and the results were averaged across the runs. A value function step size of 2e-5 and an average-reward \(\eta\) of 1.0 yielded the best results and were used to generate the results displayed in Figure \ref{fig_results}\textcolor{mylightblue}{a)}.

For the differential algorithm with smallest-magnitude explicit TD error average-reward updates, we tested every combination of the value function step size, \(\alpha\in\{\text{2e-6, 2e-5, 2e-4, 2e-3}\}\), with the average-reward step size, \(\eta\alpha\), where \(\eta\in\{\text{1e-2, 1e-1, 1.0, 10.0, 100.0}\}\), for a total of 20 unique combinations. Each combination was run 3 times using different random seeds for 1M steps each, and the results were averaged across the runs. A value function step size of 2e-5 and an average-reward \(\eta\) of 10.0 yielded the best results and were used to generate the results in Figure \ref{fig_results}\textcolor{mylightblue}{a)}.\\

\textbf{Pong Experiment:}

In the Pong experiment, we compared a value-based reward centering algorithm that utilizes implicit TD error average-reward updates (Algorithm \ref{alg_2}) against a value-based reward centering algorithm that utilizes explicit TD error average-reward updates (Algorithm \ref{alg_6}).

For both algorithms, we used the Adam optimizer with the default hyperparameters on the smooth L1 loss (Equation \eqref{eq_smoothl1}) with a \(\lambda\) of 1.0. We used a replay buffer with a minimum size of 1e2, a maximum size of 1e5, and a sample batch size of 32. We used a target network, and updated it using Polyak averaging with a fixed step size of 0.005. We used an \(\varepsilon\)-greedy policy with \(\varepsilon=\) 0.1. We used an initial average-reward estimate guess of 0.25. We used a discount factor of 0.99. We used the value network shown in Section \ref{exp_estimates}.

For the reward centering algorithm with implicit TD error average-reward updates, we tested every combination of the value function step size, \(\alpha\in\{\text{2e-6, 2e-5, 2e-4, 2e-3}\}\), with the average-reward step size, \(\eta\alpha\), where \(\eta\in\{\text{1e-2, 1e-1, 1.0}\), \(\text{10.0, 100.0}\}\), for a total of 20 unique combinations. Each combination was run 3 times using different random seeds for 1M steps each, and the results were averaged across the runs. A value function step size of 2e-5 and an average-reward \(\eta\) of 1e-2 yielded the best results and were used to generate the results displayed in Figure \ref{fig_results}\textcolor{mylightblue}{b)}.

For the reward centering algorithm with explicit TD error average-reward updates, we tested every combination of the value function step size, \(\alpha\in\{\text{2e-6, 2e-5, 2e-4, 2e-3}\}\), with the average-reward step size, \(\eta\alpha\), where \(\eta\in\{\text{1e-2, 1e-1, 1.0}\), \(\text{10.0, 100.0}\}\), for a total of 20 unique combinations. Each combination was run 3 times using different random seeds for 1M steps each, and the results were averaged across the runs. A value function step size of 2e-4 and an average-reward \(\eta\) of 1e-2 yielded the best results and were used to generate the results displayed in Figure \ref{fig_results}\textcolor{mylightblue}{b)}.\\

\textbf{HalfCheetah Experiment:}

In the HalfCheetah experiment, we compared an A2C algorithm that utilizes implicit TD error advantage estimates (Algorithm \ref{alg_3}) vs. an A2C algorithm that utilizes explicit TD error advantage estimates (Algorithm \ref{alg_7}).

\newpage

For both algorithms, we used the Adam optimizer with the default hyperparameters. We used a parameterized value function with the smooth L1 loss (Equation \eqref{eq_smoothl1} with a \(\lambda\) of 1.0). We used a parameterized softmax policy. We used a replay buffer with a minimum size of 1e2, a maximum size of 1e5, and a sample batch size of 32. We used a target network, and updated it using Polyak averaging with a fixed step size of 0.005. We used a discount factor of 0.99. We used the following critic (i.e., value) and actor (i.e., policy) networks with \textit{hidden\_dim} = 256:

Critic network:
\begin{lstlisting}[language=Python, numbers=none]
import torch

class ValueNetwork(torch.nn.Module):
    def __init__(self, state_dim, action_dim, hidden_dim):
        super(ValueNetwork, self).__init__()
        self.linear1 = torch.nn.Linear(state_dim, hidden_dim)
        self.linear2 = torch.nn.Linear(hidden_dim, hidden_dim)
        self.linear3 = torch.nn.Linear(hidden_dim, action_dim)

    def forward(self, x):
        x = torch.nn.functional.relu(self.linear1(x))
        x = torch.nn.functional.relu(self.linear2(x))
        return self.linear3(x)
\end{lstlisting}

Actor network:
\begin{lstlisting}[language=Python, numbers=none]
import torch

class PolicyNetwork(torch.nn.Module):
    def __init__(self, state_dim, action_dim, hidden_dim, max_action, 
                 log_std_min=-20, log_std_max=2):
        super(PolicyNetwork, self).__init__()

        self.log_std_min = log_std_min
        self.log_std_max = log_std_max

        self.linear1 = torch.nn.Linear(state_dim, hidden_dim)
        self.linear2 = torch.nn.Linear(hidden_dim, hidden_dim)

        self.mean = torch.nn.Linear(hidden_dim, action_dim)
        self.log_std = torch.nn.Linear(hidden_dim, action_dim)

        self.max_action = max_action

    def forward(self, state):
        a = torch.nn.functional.relu(self.linear1(state))
        a = torch.nn.functional.relu(self.linear2(a))
        mean = self.mean(a)
        log_std = self.log_std(a)
        log_std = torch.clamp(log_std, self.log_std_min, self.log_std_max)
        return mean, log_std.exp()

    def sample_action(self, state):
        mean, std = self.forward(state)
        normal = torch.distributions.Normal(mean, std)
        x_t = normal.rsample()
        action = torch.tanh(x_t)

        return action * self.max_action
\end{lstlisting}

For the A2C algorithm with implicit TD error advantage estimates, we tested every combination of the value function step size, \(\alpha\in\{\text{2e-6, 2e-5, 2e-4, 2e-3}\}\), with the policy step size, \(\eta\alpha\), where \(\eta\in\{\text{1e-2, 1e-1, 1.0, 10.0, 100.0}\}\), for a total of 20 unique combinations. Each combination was run 3 times using different random seeds for 1M steps each, and the results were averaged across the runs. A value function step size of 2e-4 and a policy \(\eta\) of 1e-2 yielded the best results and were used to generate the results displayed in Figure \ref{fig_results}\textcolor{mylightblue}{c)}.

For the A2C algorithm with explicit TD error advantage estimates, we tested every combination of the value function step size, \(\alpha\in\{\text{2e-6, 2e-5, 2e-4, 2e-3}\}\), with the policy step size, \(\eta\alpha\), where \(\eta\in\{\text{1e-2, 1e-1, 1.0, 10.0, 100.0}\}\), for a total of 20 unique combinations. Each combination was run 3 times using different random seeds for 1M steps each, and the results were averaged across the runs. A value function step size of 2e-4 and a policy \(\eta\) of 1e-2 yielded the best results and were used to generate the results displayed in Figure \ref{fig_results}\textcolor{mylightblue}{c)}.

\subsection{Baselines}
\label{baseline_algorithms}
Below is the pseudocode for the baseline algorithms used for comparison in our experiments:

\begin{algorithm}
   \caption{Differential DQN}
   \label{alg_4}
\begin{algorithmic}
    \STATE {\bfseries Input:} a differentiable state-action value function parameterization: \(\hat{q}(s, a, \boldsymbol{w})\) (with target network \(\hat{q}_{_T}(s, a, \boldsymbol{w_{_T}})\)), the policy \(\pi\) to be used (e.g., \(\varepsilon\)-greedy)
    \STATE {\bfseries Algorithm parameters:} step size parameters \(\{\alpha\), \(\eta\}\), loss parameter \(\lambda\)
    \STATE Initialize state-action value weights \(\boldsymbol{w}, \boldsymbol{w_{_T}} \in \mathbb{R}^{d}\) arbitrarily (e.g. to \(\boldsymbol{0}\))    
    \STATE Initialize \(\bar{R}\) arbitrarily (e.g. to zero)
    \STATE Obtain initial \(S\)
    \WHILE{still time to train}
        \STATE \(A \leftarrow\) action given by \(\pi\) for \(S\)
        \STATE Take action \(A\), observe \(R, S'\)
        \STATE Store \((S, A, R, S')\) in replay buffer
        \IF {time to update estimates}
        \STATE Sample a minibatch of \(B\) transitions from replay buffer: \(\{(S_b, A_b, R_b, S_b')\}_{b=1}^{B}\)
        \STATE For each \(b\)-th transition: \(\delta^e_b = R_b - \bar{R} + \max_a \hat{q}_{_T}(S_b', a, \boldsymbol{w_{_T}}) - \hat{q}(S_b, A_b, \boldsymbol{w})\)
        \STATE \(\ell = \frac{1}{B}\sum_{b=1}^{B}L^{\lambda}(\delta^e_b)\) (See Equation \eqref{eq_smoothl1})
        \STATE \(\boldsymbol{w} = \boldsymbol{w} - \alpha\frac{\partial \ell}{\partial \boldsymbol{w}}\)
        \STATE \(\bar{R} = \bar{R} + \eta\alpha\frac{1}{B}\sum_{b=1}^{B}\delta^e_b\)
        \STATE Update \(\boldsymbol{w_{_T}}\) as needed (e.g. using Polyak averaging)
        \ENDIF
        \STATE \(S = S'\)
    \ENDWHILE
    \STATE return \(\boldsymbol{w}\)
\end{algorithmic}
\end{algorithm}

\begin{algorithm}
   \caption{A2C \citep{Mnih2016-tt}}
   \label{alg_7}
\begin{algorithmic}
    \STATE {\bfseries Input:} a differentiable state-value function parameterization \(\hat{v}(s, \boldsymbol{w})\), a differentiable policy parameterization \(\pi(a \mid s, \boldsymbol{u})\)
    \STATE {\bfseries Algorithm parameters:} discount factor \(\gamma\), step size parameters \(\{\alpha\), \(\eta\}\)
    \STATE Initialize state-value weights \(\boldsymbol{w} \in \mathbb{R}^{d}\) and policy weights \(\boldsymbol{u} \in \mathbb{R}^{d'}\) (e.g. to \(\boldsymbol{0}\))
    \STATE Obtain initial \(S\)
    \WHILE{still time to train}
        \STATE \(A \sim \pi(\cdot \mid S, \boldsymbol{u})\)
        \STATE Take action \(A\), observe \(R, S'\)
        \STATE \(\delta^e = R + \gamma \hat{v}(S', \boldsymbol{w}) - \hat{v}(S, \boldsymbol{w})\)
        \STATE \(\boldsymbol{w} = \boldsymbol{w} + \alpha\delta^e\nabla\hat{v}(S, \boldsymbol{w})\)
        \STATE \(\boldsymbol{u} = \boldsymbol{u} + \eta\alpha\delta^e\nabla \ln \pi(A \mid S, \boldsymbol{u})\)
        \STATE \(S = S'\)
    \ENDWHILE
    \STATE return \(\boldsymbol{w}\) and \(\boldsymbol{u}\)
\end{algorithmic}
\end{algorithm}

\begin{algorithm}
   \caption{Differential DQN (Smallest-Magnitude Baseline from \citet{Rojas2026-dv})}
   \label{alg_5}
\begin{algorithmic}
    \STATE {\bfseries Input:} a differentiable state-action value function parameterization: \(\hat{q}(s, a, \boldsymbol{w})\) (with target network \(\hat{q}_{_T}(s, a, \boldsymbol{w_{_T}})\)), the policy \(\pi\) to be used (e.g., \(\varepsilon\)-greedy)
    \STATE {\bfseries Algorithm parameters:} step size parameters \(\{\alpha\), \(\eta\}\), loss parameter \(\lambda\)
    \STATE Initialize state-action value weights \(\boldsymbol{w}, \boldsymbol{w_{_T}} \in \mathbb{R}^{d}\) arbitrarily (e.g. to \(\boldsymbol{0}\))    
    \STATE Initialize \(\bar{R}\) arbitrarily (e.g. to zero)
    \STATE Obtain initial \(S\)
    \WHILE{still time to train}
        \STATE \(A \leftarrow\) action given by \(\pi\) for \(S\)
        \STATE Take action \(A\), observe \(R, S'\)
        \STATE Store \((S, A, R, S')\) in replay buffer
        \IF {time to update estimates}
        \STATE Sample a minibatch of \(B\) transitions from replay buffer: \(\{(S_b, A_b, R_b, S_b')\}_{b=1}^{B}\)
        \STATE For each \(b\)-th transition: \(\delta^e_b = R_b - \bar{R} + \max_a \hat{q}_{_T}(S_b', a, \boldsymbol{w_{_T}}) - \hat{q}(S_b, A_b, \boldsymbol{w})\)
        \STATE \(\ell = \frac{1}{B}\sum_{b=1}^{B}L^{\lambda}(\delta^e_b)\) (See Equation \eqref{eq_smoothl1})
        \STATE \(\boldsymbol{w} = \boldsymbol{w} - \alpha\frac{\partial \ell}{\partial \boldsymbol{w}}\)
        \STATE Use the smallest-magnitude explicit TD error for the average-reward update: 
        \STATE \quad \(y = \arg\min_{b \in \{1,\dots,B\}} |\delta^e_b|\) (breaking ties arbitrarily)
        \STATE \quad \(\bar{R} = \bar{R} + \eta\alpha\delta^e_y\)
        \STATE Update \(\boldsymbol{w_{_T}}\) as needed (e.g. using Polyak averaging)
        \ENDIF
        \STATE \(S = S'\)
    \ENDWHILE
    \STATE return \(\boldsymbol{w}\)
\end{algorithmic}
\end{algorithm}

\begin{algorithm}
   \caption{Value-Based Reward Centering DQN \citep{Naik2024-km}}
   \label{alg_6}
\begin{algorithmic}
    \STATE {\bfseries Input:} a differentiable state-action value function parameterization: \(\hat{q}(s, a, \boldsymbol{w})\) (with target network \(\hat{q}_{_T}(s, a, \boldsymbol{w_{_T}})\)), the policy \(\pi\) to be used (e.g., \(\varepsilon\)-greedy)
    \STATE {\bfseries Algorithm parameters:} discount factor \(\gamma\), step size parameters \(\{\alpha\), \(\eta\}\), loss parameter \(\lambda\)
    \STATE Initialize state-action value weights \(\boldsymbol{w}, \boldsymbol{w_{_T}} \in \mathbb{R}^{d}\) arbitrarily (e.g. to \(\boldsymbol{0}\))    
    \STATE Initialize \(\bar{R}\) arbitrarily (e.g. to zero)
    \STATE Obtain initial \(S\)
    \WHILE{still time to train}
        \STATE \(A \leftarrow\) action given by \(\pi\) for \(S\)
        \STATE Take action \(A\), observe \(R, S'\)
        \STATE Store \((S, A, R, S')\) in replay buffer
        \IF {time to update estimates}
        \STATE Sample a minibatch of \(B\) transitions from replay buffer: \(\{(S_b, A_b, R_b, S_b')\}_{b=1}^{B}\)
        \STATE For each \(b\)-th transition: \(\delta^e_b = R_b - \bar{R} + \gamma \max_a \hat{q}_{_T}(S_b', a, \boldsymbol{w_{_T}}) - \hat{q}(S_b, A_b, \boldsymbol{w})\)
        \STATE \(\ell = \frac{1}{B}\sum_{b=1}^{B}L^{\lambda}(\delta^e_b)\) (See Equation \eqref{eq_smoothl1})
        \STATE \(\boldsymbol{w} = \boldsymbol{w} - \alpha\frac{\partial \ell}{\partial \boldsymbol{w}}\)
        \STATE \(\bar{R} = \bar{R} + \eta\alpha\frac{1}{B}\sum_{b=1}^{B}\delta^e_b\)
        \STATE Update \(\boldsymbol{w_{_T}}\) as needed (e.g. using Polyak averaging)
        \ENDIF
        \STATE \(S = S'\)
    \ENDWHILE
    \STATE return \(\boldsymbol{w}\)
\end{algorithmic}
\end{algorithm}


\end{document}